%% file: main.tex
\documentclass[11pt]{article}
\usepackage[T1]{fontenc}
\usepackage[utf8]{inputenc}
\usepackage{lmodern}
\usepackage[margin=1in]{geometry}
\usepackage{amsmath,amssymb,amsthm,mathtools}
\usepackage{microtype}
\usepackage[numbers,sort&compress]{natbib}
\usepackage[hidelinks]{hyperref}
\hypersetup{pdftitle={Polynomial Learning of Positive ReLU Networks under Gaussian Inputs},pdfauthor={Zhijie Zheng}}
\newtheorem{theorem}{Theorem}[section]
\newtheorem{lemma}[theorem]{Lemma}
\newtheorem{proposition}[theorem]{Proposition}
\newtheorem{corollary}[theorem]{Corollary}
\theoremstyle{definition}

\theoremstyle{remark}
\newtheorem{remark}[theorem]{Remark}
\newcommand{\R}{\mathbb{R}}

\newcommand{\E}{\mathbb{E}}
\newcommand{\Prob}{\mathbb{P}}
\newcommand{\N}{\mathcal{N}}
\newcommand{\relu}{\operatorname{ReLU}}
\newcommand{\poly}{\operatorname{poly}}
\newcommand{\dist}{\operatorname{dist}}
\newcommand{\range}{\operatorname{range}}
\newcommand{\tr}{\operatorname{tr}}
\newcommand{\Sym}{\operatorname{Sym}}

\newcommand{\op}{\mathrm{op}}
\newcommand{\Frob}{\mathrm{F}}
\newcommand{\eps}{\varepsilon}
\newcommand{\ii}{\mathrm{i}}
\newcommand{\norm}[1]{\left\lVert#1\right\rVert}
\newcommand{\abs}[1]{\left\lvert#1\right\rvert}
\newcommand{\inner}[2]{\left\langle#1,#2\right\rangle}
\newcommand{\sphere}{\mathbb{S}^{d-1}}
\newcommand{\Coarse}{\mathsf{CoarseCover}}
\newcommand{\Prune}{\mathsf{Prune}}
\newcommand{\coeffset}{\Delta_{\leq1}}
\newcommand{\Fk}{\mathcal{F}_k}

\allowdisplaybreaks[1]
\title{Polynomial Learning of Positive ReLU Networks\\under Gaussian Inputs}
\author{Zhijie Zheng\\Peking University\\\texttt{2501210512@stu.pku.edu.cn}\\
\small\href{https://orcid.org/0009-0006-7610-0691}{ORCID: 0009-0006-7610-0691}}
\date{}
\begin{document}
\maketitle

\begin{abstract}
We study positive, bias-free, one-hidden-layer ReLU networks under standard Gaussian inputs. We obtain algorithms with polynomial sample and real-arithmetic complexity, without separation or conditioning assumptions on the hidden directions. When the conditional label mean is a width-$k$ network with total output weight at most one, the learner achieves squared prediction error $\eps$ with output width at most $k+\min\{k,d\}$ and complexity linear in one plus the residual second moment. For arbitrary labels with bounded second moment, we obtain a positive ReLU network of width $O(k\log(2/\eps))$. Its risk is bounded by the irreducible noise variance plus $Ck^3\log^{9/2}(2k/\eps)$ times the squared distance of the conditional mean to the target class, plus $\eps$. The proofs combine implicit Hermite moments, orthogonal tensor-train compression, and a direction decoder whose probability of returning a direction far from every atom is bounded linearly in the compression and estimation errors. Exact Gaussian localization refines the resulting direction covers. In the agnostic setting, a cone-constrained convex program with a scale-dependent penalty produces a positive network; in the conditional-mean setting, exact sparsification yields the smaller width. We also establish projective direction recovery and confidence amplification with sample complexity proportional to the logarithm of the reciprocal failure probability.
\end{abstract}

\section{Introduction}
\label{sec:introduction}
Learning neural networks in time jointly polynomial in dimension, width, and inverse accuracy is a central problem in computational learning theory. Positive sums of ReLU activations have useful structure even when their hidden directions are arbitrarily close or linearly dependent. After normalization, their even Hermite moments are positive semidefinite mixtures of tensor powers. To exploit this structure, one must represent high-order moments efficiently and turn coarse estimates of the hidden directions into accurate predictors.

We consider two noise models. In the conditional-mean model, $\E[Y\mid X]=f(X)$ for a target positive ReLU network $f$, with input-dependent residual noise of finite unconditional variance. In the agnostic model, the conditional mean $g=\E[Y\mid X]$ may lie outside the target class, and the benchmark is its best approximation by a network in that class. The mean approximation error $g-f$ vanishes in the first model and perturbs the population moments in the second. Separating this error from label noise leads to an agnostic risk bound in which the irreducible variance has coefficient one.

Both algorithms rely on a complete and sound cover of the hidden directions: every sufficiently heavy atom is approximated, and every returned center is close to an atom of the comparator network. Ordered tensor-train compression maintains exact orthogonality and bounds the total weighted squared projection error by the sum of the errors at individual stages. Successive Gaussian contractions concentrate the posterior distribution on nearby projective directions. Perturbation bounds and a frequency test then yield a robust cover with at most $k$ centers. Gaussian localization magnifies neighborhoods of these centers, and an inverse-map estimate converts coarse estimates in the localized coordinates into finer estimates in the original coordinates.

In the conditional-mean model, the direction estimates can be refined to arbitrary accuracy. Convex coefficient fitting followed by exact sparsification then gives a network of width $k+\min\{k,d\}$. In the agnostic model, each atom is approximated at a scale determined by its weight and the mean approximation error. A tangent approximation with cubic squared error under Gaussian inputs leads to a convex fitting problem. A scale-dependent penalty controls the error when the fitted predictor is converted into a positive ReLU network. The resulting network satisfies the same positivity and total weight constraints as the target class, with a logarithmic increase in width.

\paragraph{Relation to prior work.}
\citet{diakonikolas2025implicit} developed implicit high-order moment estimation and recursive compression for latent-variable models. Their positive-ReLU learner has complexity $\poly(d,k)2^{\poly(1/\eps)}$. We build on their moment framework and complex Hermite queries, using an orthogonal tensor-train representation and estimates that require only a finite second moment of the labels. Polynomial reweighting and spectral rounding also appear in tensor decomposition~\citep{ma2016polynomial}. Appendix~\ref{sec:alternative} gives a cover based on complex Gaussian reweighting for the conditional-mean model. The sequential decoder in Section~\ref{sec:robust-cover} also accommodates perturbations of the conditional mean.

Exact Gaussian soft localization is an established change-of-measure technique; the sampling identity below specializes Lemma~14 of \citet{diakonikolas2021agnostic}. The present analysis combines it with a sound multi-atom cover, a local inverse map, and pruning that controls the number of active centers at every scale. For agnostic learning, the analysis relates the increase in mean error under localization to the cubic error of the tangent approximation.

For a single arbitrarily biased ReLU, \citet{guo2025agnostic} obtain constant-factor agnostic learning under Gaussian marginals. Our setting allows variable width but assumes zero hidden biases, and our approximation factor depends on width. Related results for multi-index models cover different structural assumptions: \citet{diakonikolas2025finite} treat finite-range models, while \citet{diakonikolas2025real} give algorithms for real-valued models with exponential dependence on latent dimension and accuracy. The positive-sum structure considered here allows polynomial dependence on dimension, width, and inverse accuracy simultaneously.

Cone-constrained convex formulations of neural networks and conversions between gated and ReLU models have been studied by \citet{mishkin2022convex} and \citet{dwaraknath2023ntk}. We use circular cones centered at direction estimates from several localization scales. A penalty weighted by these scales bounds the population error introduced by conversion to a positive network.

\section{Models and main results}
\label{sec:model}
Throughout, $d,k\geq1$ are integers, $X\sim\N(0,I_d)$, and $\relu(z)=\max\{z,0\}$. Define
\begin{equation}
\Fk=\left\{f(x)=\sum_{i=1}^{k_0}a_i\relu(u_i^\top x):
 k_0\leq k,\ a_i\geq0,\ \sum_i a_i\leq1,\ u_i\in\sphere\right\}.
\label{eq:class}
\end{equation}
We refer to $u_i$ as atoms and omit zero-weight atoms. No separation or lower singular-value bound is imposed. The norm $\norm{h}_2$ denotes the $L^2(\N(0,I_d))$ norm unless a finite-dimensional space is specified. Positive homogeneity allows any nonzero hidden vector to be normalized by transferring its norm to the output coefficient. Every $f\in\Fk$ is continuous, convex, nonnegative, and $1$-Lipschitz, and $\E f(X)^2\leq1/2$.

For labels with $\E Y^2<\infty$, let
\begin{equation}
 g(X)=\E[Y\mid X],\qquad
 \mathcal V=\E[(Y-g(X))^2],\qquad
 \mathcal A_k=\inf_{f\in\Fk}\norm{g-f}_2^2.
 \label{eq:noise-decomposition}
\end{equation}
The risk $\mathcal R(h)=\E[(Y-h(X))^2]$ satisfies the exact orthogonal decomposition
\begin{equation}
 \mathcal R(h)=\mathcal V+\norm{g-h}_2^2,
 \qquad \mathrm{OPT}_k:=\inf_{f\in\Fk}\mathcal R(f)
                       =\mathcal V+\mathcal A_k.
 \label{eq:risk-decomposition}
\end{equation}
The infimum is attained. After padding with zero coefficients to $k$ units, the parameter space is the product of a compact coefficient simplex and unit spheres. The parameterization is continuous in Gaussian $L^2$, since the ReLU Lipschitz inequality gives
$\norm{\relu(u^\top X)-\relu(v^\top X)}_2\leq\norm{u-v}_2$.

\begin{theorem}[Agnostic learning by positive ReLU networks]
\label{thm:agnostic}
Suppose $B\geq1$ is a known bound with $\E Y^2\leq B$. For $0<\eps\leq1$ and $0<\delta<1$, there is a randomized algorithm using independent samples that, with probability at least $1-\delta$, returns
\begin{equation}
 H(x)=\sum_{r=1}^{m}\alpha_r\relu(v_r^\top x),\qquad
 m\leq Ck\log(2/\eps),\quad \alpha_r\geq0,\quad
 \sum_r\alpha_r\leq1,\quad v_r\in\sphere,
 \label{eq:agnostic-output}
\end{equation}
satisfying
\begin{equation}
 \mathcal R(H)\leq\mathcal V+
 Ck^3\log^{9/2}(2k/\eps)\mathcal A_k+\eps.
 \label{eq:agnostic-main}
\end{equation}
It uses at most $\poly(d,k,B,\eps^{-1},\log(2/\delta))$ samples and
real-arithmetic operations, with universal exponents.
The algorithm requires no knowledge of $\mathcal V$, $\mathcal A_k$, or $\mathrm{OPT}_k$.
The coefficient and width constraints hold for every output.
\end{theorem}

In particular, $H$ is continuous, convex, nonnegative, and $1$-Lipschitz. After increasing a universal constant, \eqref{eq:agnostic-main} implies the polynomial-factor agnostic guarantee
\begin{equation}
 \mathcal R(H)\leq Ck^3\log^5(2k/\eps)\mathrm{OPT}_k+\eps.
 \label{eq:original-agnostic}
\end{equation}
It also gives $\norm{H-g}_2^2\leq\eps$ when $g\in\Fk$, even with irreducible label noise. The following result provides a smaller output width and a sharper dependence on that noise in the conditional-mean setting.

\begin{theorem}[Learning a positive ReLU conditional mean]
\label{thm:conditional}
Suppose independent examples satisfy
\begin{equation}
 X\sim\N(0,I_d),\qquad \E[Y\mid X]=f(X),\quad f\in\Fk,
 \qquad \E[(Y-f(X))^2]\leq\sigma^2,
 \label{eq:conditional-model}
\end{equation}
where a bound $\sigma^2\geq0$ is supplied. For $\eps,\delta\in(0,1/10)$, a randomized algorithm returns, with probability at least $1-\delta$,
\begin{equation}
 h(x)=\sum_{j=1}^{m}c_j\relu(v_j^\top x),\qquad
 m\leq k+\min\{k,d\},\quad c_j\geq0,\quad
 \sum_jc_j\leq1,\quad v_j\in\sphere,
 \label{eq:conditional-output}
\end{equation}
with $\norm{h-f}_2^2\leq\eps$, using at most
\begin{equation}
 (1+\sigma^2)\poly(d,k,\eps^{-1},\log(1/\delta))
 \label{eq:conditional-cost}
\end{equation}
samples and real-arithmetic operations. The coefficient and width constraints hold for every output. If the final direction matrix has $s\leq k$ columns and rank $r$, the output width is at most $s+r$. Evaluating $h$ requires $O(dk)$ basic operations.
\end{theorem}

The conditional variance may be unbounded as a function of $X$; only the unconditional residual second moment is bounded. Section~\ref{sec:conditional} also gives parameterized projective direction recovery and reduces the confidence dependence in the sample count to a single factor $\log(1/\delta)$.

\paragraph{Computational model.}
Complexity is measured in a real-arithmetic model allowing exact real and complex arithmetic, comparisons, square roots, Gaussian sampling, and symmetric eigendecompositions. These operations suffice for the agnostic algorithm. Uniform sampling and the localization acceptance test can be implemented using Gaussian squares, as described in Section~\ref{sec:localization}. The conditional-mean algorithm additionally evaluates the Gaussian ReLU kernel, including its inverse-cosine term; the construction in Appendix~\ref{sec:alternative} also evaluates exponentials. Adaptive estimation uses independent samples at each stage. Prescribed polynomial upper bounds on sample counts and matrix and list sizes ensure polynomial running time for every outcome. If a bound would be exceeded, the algorithm terminates with the zero network or the empty direction set.

\paragraph{Notation.}
For unit vectors define the projective metric
\begin{equation}
 D(u,v)=\min\{\norm{u-v}_2,\norm{u+v}_2\}
       =\sqrt{2(1-\abs{u^\top v})}.
 \label{eq:projective}
\end{equation}
This is a metric on directions modulo sign. We fix $\theta=10^{-3}$ throughout the cover and refinement arguments. Constants denoted by $C,c$ are universal and may change from line to line. Sample counts are rounded up to integers. Logarithmic factors in complexity bounds may be replaced by $\log(2+\cdot)$ to keep them positive. Tensor norms and inner products are Frobenius norms and inner products; the polynomial contractions are defined in Section~\ref{sec:moments}.

\input{moment_cover}
\input{localization}
\input{conditional}
\input{agnostic}

\section{Discussion}
The agnostic guarantee separates irreducible label noise from approximation error in the conditional mean. Only the approximation error is multiplied by a factor depending on width and accuracy. When the conditional mean belongs to the target class, a specialized algorithm achieves width $k+\min\{k,d\}$ and complexity linear in one plus the residual second moment. Its sparsification argument uses pairs of ReLU units with opposite orientations along the same direction. The vectors obtained from the agnostic convex program need not occur in such pairs, so the same sparsification argument does not directly apply to the agnostic output.

Positivity is essential to the positive semidefinite moment representation and the decoder's soundness argument. The Gaussian assumption provides the Hermite identities, exact localization, and cubic tangent approximation, while the absence of hidden biases makes the network class invariant under the localization maps. Natural directions for further work include reducing the polynomial exponents and extending the analysis to signed output weights, hidden biases, or non-Gaussian inputs.

\appendix
\input{alternative_cover}

\clearpage
\bibliographystyle{plainnat}
\bibliography{references}
\end{document}

%% file: moment_cover.tex
\section{Observable moments and ordered tensor compression}
\label{sec:moments}

We construct a direction cover whose accuracy depends on the population
discrepancy
$\zeta=\norm{g-f}_{L^2}$, where $g(x)=\E[Y\mid X=x]$ and
$f(x)=\sum_i a_i\relu(u_i^\top x)\in\mathcal F_k$ is a fixed comparator.
The sample and operation bounds improve when $g=f$.
We omit zero-weight atoms throughout and allow the atom set and output
list to be empty.

Our moment estimates build on the implicit-moment framework of
\citet{diakonikolas2025implicit}. We equip ordered tensor spaces with
the Frobenius inner product. For a real order-$q$ tensor
$T$, write
\[
 p_T(z)=\sum_{i_1,\ldots,i_q}T_{i_1\cdots i_q}z_{i_1}\cdots z_{i_q}.
\]
The polynomial depends only on $S=\Sym T$, where symmetrization is an
orthogonal projection. The tensor-train representation allows us to
evaluate $p_T$ without forming its $d^q$ coefficients.
Polynomial and Hermite contractions are bilinear in the tensor
coefficients, even at complex arguments; norms and adjoints use the
Hermitian convention.

\subsection{Hermite moments and finite-variance estimation}

Let $\mathsf H_q$ be the order-$q$ probabilists' Hermite tensor, and let
$Z$ be an independent standard Gaussian vector. The generating function
gives
\begin{equation}
 \E_Z p_T(X+\ii Z)=\inner{T}{\mathsf H_q(X)},\qquad
 \left\lVert\frac{\inner{T}{\mathsf H_q}}{\sqrt{q!}}\right\rVert_{L^2}
 =\norm{\Sym T}_{\Frob}\leq\norm{T}_{\Frob}.
 \label{eq:ordered-hermite}
\end{equation}
For $q=2j\geq2$, put
\begin{equation}
 c_{2j}=\E\left[\relu(G)\frac{\mathsf H_{2j}(G)}{\sqrt{(2j)!}}\right]
 =\frac{(-1)^{j-1}(2j-3)!!}{\sqrt{2\pi}\sqrt{(2j)!}},
 \qquad |c_{2j}|^{-1}\leq C(j+1)^{5/4}.
 \label{eq:merged-hermite-coefficient}
\end{equation}
Here $G\sim\N(0,1)$ and $(-1)!!=1$. Integration by parts, using
$\relu''=\delta_0$, gives the coefficient; central-binomial-coefficient
estimates give $|c_{2j}|\asymp(j+1)^{-5/4}$. Gaussian conditioning
then yields
\begin{equation}
 \frac{\E[f(X)\inner{T}{\mathsf H_{2j}(X)}]}
      {c_{2j}\sqrt{(2j)!}}
 =\sum_i a_i\inner{T}{u_i^{\otimes 2j}}.
 \label{eq:merged-atomic-moment}
\end{equation}

\begin{lemma}[Gaussian moment and clipping bounds]
\label{lem:moment-signal}
Identity \eqref{eq:ordered-hermite} also holds for complex tensors,
with Hermitian norms. For every symmetric, possibly complex order-$q$
tensor $S$ and every real unit vector $u$,
\begin{align}
 \E|p_S(X+\ii Z)|^2
 &=2^q q!\norm{S}_{\Frob}^2,
 \label{eq:merged-complex-wick}\\
 \E[\norm{X}^2|p_S(X+\ii Z)|^2]
 &=(d+q)2^q q!\norm{S}_{\Frob}^2,
 \label{eq:merged-radial-wick}\\
 \E[(u^\top X)^2|p_S(X+\ii Z)|^2]
 &=2^q q!\left(\norm{S}_{\Frob}^2
       +q\norm{S\mathbin{\lrcorner}u}_{\Frob}^2\right),
 \label{eq:merged-directional-wick}\\
 \E\left[f(X)^2\frac{|p_S(X+\ii Z)|^2}{q!}\right]
 &\leq\frac{q+1}{2}\,2^q\norm{S}_{\Frob}^2.
 \label{eq:merged-full-signal-bound}
\end{align}
Let $T$ be a real order-$q$ tensor with $\norm{T}_{\Frob}\leq1$, and set
$S=\Sym T$ and $\psi=\operatorname{Re}p_T(X+\ii Z)/\sqrt{q!}$. Then
\begin{align}
 \E\psi^2&\leq2^q,\qquad
 \E\psi^4\leq2^{4q}, \label{eq:merged-query-moments}\\
 \E[f(X)^2\psi^2]&\leq\frac{q+1}{2}\,2^q.
 \label{eq:merged-signal-moment}
\end{align}
In the conditional-mean model $Y=f(X)+\xi$, $\E[\xi\mid X]=0$,
$\E\xi^2\leq\sigma^2$, clipping $\psi$ at level $M>0$ gives
\begin{align}
 \abs{\E[Y\operatorname{clip}_M(\psi)]-\E[Y\psi]}
 &\leq \frac{2^q\sqrt{(q+1)/2}}{M},
 \label{eq:merged-signal-clipping}\\
 \E[Y^2\operatorname{clip}_M(\psi)^2]
 &\leq M^2(1/2+\sigma^2).
 \label{eq:merged-linear-noise}
\end{align}
\end{lemma}

\begin{proof}
Complex Gaussian pairings give, for a symmetric order-$q$ tensor $S$,
\[
 \E|p_S(X+\ii Z)|^2=2^q q!\norm{S}_{\Frob}^2.
\]
Each holomorphic coordinate pairs with one conjugate coordinate,
contributing a factor of two; symmetry makes the $q!$ contractions
equal. The Hermite identities extend to complex coefficients by
linearity and polarization.
Apply the same identity to
$p_S^2=p_{\Sym(S\otimes S)}$, and use
$\norm{\Sym(S\otimes S)}_{\Frob}\leq\norm{S}_{\Frob}^2$. Dividing by
$(q!)^2$ bounds the fourth moment by
$2^{2q}\binom{2q}{q}\leq2^{4q}$.

For a real unit vector $u$, circular symmetry and the same pairing
identity give
\[
 \E[(u^\top X)^2|p_S(X+\ii Z)|^2]
 =2^q q!\left(\norm{S}_{\Frob}^2
       +q\norm{S\mathbin{\lrcorner}u}_{\Frob}^2\right)
 \leq(q+1)2^q q!\norm{S}_{\Frob}^2.
\]
To verify the equality, apply the pairing identity to
$\Sym(u\otimes S)$ and use
\[
 \norm{\Sym(u\otimes S)}_{\Frob}^2
 =\frac{\norm{S}_{\Frob}^2+
          q\norm{S\mathbin{\lrcorner}u}_{\Frob}^2}{q+1}.
\]
Summing \eqref{eq:merged-directional-wick} over an orthonormal basis
of $\R^d$ proves \eqref{eq:merged-radial-wick}, since the squared
norms of the coordinate contractions sum to $\norm{S}_{\Frob}^2$.
Weighted Cauchy--Schwarz gives
$f(X)^2\leq\sum_i a_i\relu(u_i^\top X)^2$. Under
$(X,Z)\mapsto(-X,-Z)$, $|p_S(X+\ii Z)|^2$ is unchanged. Replacing
$(u_i^\top X)^2$ by $\relu(u_i^\top X)^2$ therefore halves its
expectation, proving \eqref{eq:merged-full-signal-bound} and hence
\eqref{eq:merged-signal-moment}.

Because $Z$ is independent of $(X,Y)$, the noise has zero expectation
against every clipped multiplier depending on $(X,Z)$. The clipping
bias therefore depends only on $f$. Cauchy--Schwarz and Markov's
inequality bound it by
\[
 \E[|f(X)\psi|\mathbf1_{\{|\psi|>M\}}]
 \leq\sqrt{\E[f(X)^2\psi^2]}\,
       \sqrt{\Prob(|\psi|>M)}
 \leq\frac{2^q\sqrt{(q+1)/2}}{M}.
\]
Finally, conditional zero-mean noise gives
$\E Y^2=\E f(X)^2+\E\xi^2\leq1/2+\sigma^2$.
The bounded multiplier proves \eqref{eq:merged-linear-noise}.
\end{proof}

\begin{remark}[Scaled complex multipliers]
The clipping bounds also hold for either the real or imaginary part of
$p_S(X+\ii Z)/\sqrt{q!}$ when $S$ is a complex symmetric tensor with
$\norm{S}_{\Frob}\leq M_0$. With
$V_0=2^qM_0^2$ and $V_1=(q+1)2^qM_0^2/2$, clipping at level $L>0$
gives bias at most $\sqrt{V_0V_1}/L$ and second moment at most
$L^2(1/2+\sigma^2)$. This follows from the same
Cauchy--Schwarz argument and gives the bounds for the complex estimators in
Appendix~\ref{sec:alternative}.
\end{remark}

\begin{remark}[Tensor products and Gram identities]
For symmetric tensors $S,V$ of any orders, including complex tensors,
\[
 p_S(z)p_V(z)=p_{\Sym(S\otimes V)}(z),\qquad
 \norm{\Sym(S\otimes V)}_{\Frob}
 \leq\norm{S}_{\Frob}\norm{V}_{\Frob}.
\]
For equal order $q$, polarization of
\eqref{eq:merged-complex-wick} gives
\[
 \frac{\E[p_S(X+\ii Z)\overline{p_V(X+\ii Z)}]}{2^q q!}
 =\sum_{\boldsymbol i}S_{\boldsymbol i}
                         \overline{V_{\boldsymbol i}}.
\]
The normalized Gram estimator has second moment at most
$\binom{2q}{q}\leq4^q$ when both Frobenius norms are at most one:
apply the product identity to $p_S^2$ and $p_V^2$, then
Cauchy--Schwarz. The ordered compression below uses exact
orthonormality and therefore does not require estimating these Gram
matrices.
\end{remark}

\begin{lemma}[Estimation of compressed moments]
\label{lem:query-access}
Let $V$ have $n\leq kd$ orthonormal columns in
$(\R^d)^{\otimes j}$, each with a polynomial-size tensor-train
representation. The matrix of degree-$2j$ label moments associated with
$v_\alpha\otimes v_\beta$ is
\begin{equation}
 \mathcal M(V)
 =V^\top\left(\sum_i a_i
       u_i^{\otimes j}(u_i^{\otimes j})^\top\right)V+\mathcal E,
 \qquad
 \norm{\mathcal E}_{\Frob}\leq |c_{2j}|^{-1}\zeta.
 \label{eq:merged-compressed-moment}
\end{equation}
It can be estimated to additional Frobenius error $\eta>0$, with
failure probability at most $\omega$, using
\begin{equation}
 B^{O(1)}\exp(O(j))
 \poly(d,k,\eta^{-1},\log(2/\omega))
 \label{eq:merged-general-access}
\end{equation}
samples and real operations when $\E Y^2\leq B$, $B\geq1$.
In the conditional-mean model, $\mathcal E=0$ and the corresponding
cost is
\begin{equation}
 (1+\sigma^2)\exp(O(j))
 \poly(d,k,\eta^{-1},\log(2/\omega)).
 \label{eq:merged-mean-access}
\end{equation}
These bounds include the cost of evaluating the tensor-train representations.
\end{lemma}

\begin{proof}
Write $e=g-f$. Its normalized degree-$2j$ Hermite coefficient tensor
$\E[e(X)\mathsf H_{2j}(X)]/\sqrt{(2j)!}$ has Frobenius norm at most
$\norm e_{L^2}=\zeta$, by \eqref{eq:ordered-hermite} and orthogonal
projection onto that Gaussian chaos. Flattening is an isometry, and
left and right compression by $V$ are Frobenius contractions. This
proves the dimension-independent bound
\eqref{eq:merged-compressed-moment}.

Each query tensor $v_\alpha\otimes v_\beta$ has Frobenius norm one.
For arbitrary labels, use
$\Psi=\operatorname{Re}p_T(X+\ii Z)/(c_{2j}\sqrt{(2j)!})$.
Lemma~\ref{lem:moment-signal} gives
$\E\Psi^4\leq\exp(O(j))|c_{2j}|^{-4}=:V_4$. Clipping at level $M$
changes the mean by at most
\[
 \sqrt B\left(\E[\Psi^2\mathbf1_{\{|\Psi|>M\}}]\right)^{1/2}
 \leq \frac{\sqrt{B V_4}}{M},
\]
while the clipped product with $Y$ has variance at most $BM^2$.
For scalar accuracy $b$, choose
$M=\max\{1,2\sqrt{B V_4}/b\}$ and estimate the clipped mean to
accuracy $b/2$ by median-of-means. This costs
$B^{O(1)}\exp(O(j))\poly(b^{-1},\log(2/\omega))$.

In the conditional-mean model, clip the unscaled multiplier
$\psi$ in Lemma~\ref{lem:moment-signal}. Choose
\[
 M=\max\left\{1,\frac{2^{2j+1}\sqrt{(2j+1)/2}}
                         {|c_{2j}|b}\right\},
\]
estimate its labeled mean to accuracy $|c_{2j}|b/2$, and divide by
$c_{2j}$. The bias bound is independent of $\sigma$, and the variance
is linear in $1+\sigma^2$. This proves the sharper dependence in
\eqref{eq:merged-mean-access}.

For a mean of variance at most $V$, median-of-means uses
$O((1+V/b^2)\log(2/\omega))$ samples. Set entrywise accuracy to
$\eta/n$ and distribute the failure probability among the entries. Take real
parts and symmetrize the estimated matrix; neither operation increases
its error relative to the real symmetric target. Fresh samples are
drawn after $V$ is fixed, so the same bounds hold conditionally when
the matrices are chosen adaptively.
\end{proof}

\subsection{Exact isometries in ordered tensor spaces}

\begin{lemma}[Tensor-train compression]
\label{lem:tt-compression}
For $t\geq1$, statistical precision $\eta>0$, and failure probability
$\omega\in(0,1)$, one can construct an
exact isometry $V_t$ with at most $k$ columns and a PSD matrix
$\widehat C$ of that size. Set
\[
 M_j=\sum_i a_i u_i^{\otimes j}(u_i^{\otimes j})^\top,\quad
 P_j=V_jV_j^\top,\quad
 L_j=\sum_i a_i\norm{(I-P_j)u_i^{\otimes j}}^2,
 \quad R=V_t\widehat C V_t^\top .
\]
With the prescribed success probability,
\begin{equation}
 \Gamma:=L_t+\norm{R-P_tM_tP_t}_1
 \leq C\sqrt{k}(t+1)^{9/4}(\zeta+\eta).
 \label{eq:merged-linear-leakage}
\end{equation}
Here $\norm{\cdot}_1$ denotes the trace norm. The columns of $V_t$
have tensor-train bond dimension at most $k$. The sample and operation
bounds are as in \eqref{eq:merged-general-access}, with $j$ replaced by $t$;
in the conditional-mean model they are as in
\eqref{eq:merged-mean-access}.
\end{lemma}

\begin{proof}
Start with the scalar isometry $V_0=1$. At stage $j$, form the ordered
isometry $U_j=V_{j-1}\otimes I_d$, which has at most $kd$ columns.
Use Lemma~\ref{lem:query-access} to estimate
$C_j=U_j^\top M_jU_j$ to Frobenius error
\[
 \Delta_j=\eta+|c_{2j}|^{-1}\zeta.
\]
If $r_j=\min\{k,\dim\range U_j\}$, let $R_j$ contain the leading
$r_j$ orthonormal eigenvectors of the symmetric estimate, and set
$V_j=U_jR_j$. All these maps are exact isometries, irrespective of
estimation accuracy.

Projection onto $\range U_j$ loses exactly $L_{j-1}$, because
$u_i^{\otimes j}=u_i^{\otimes(j-1)}\otimes u_i$. The matrix
$C_j$ is PSD of rank at most $r_j$. Let $P=R_jR_j^\top$ and let $P_0$
be any rank-$r_j$ orthogonal projector containing its range. Optimality
of $P$ for the estimated matrix $\widehat C_j$ implies
\begin{align*}
 \tr((I-P)C_j)
 &=\tr((P_0-P)C_j)\\
 &\leq\abs{\tr((P-P_0)(\widehat C_j-C_j))}
 \leq\sqrt{2k}\,\Delta_j.
\end{align*}
The two nested orthogonal projection errors add as squared norms.
Consequently,
\begin{equation}
 L_j\leq L_{j-1}+\sqrt{2k}\,\Delta_j,\qquad L_0=0.
 \label{eq:merged-tt-recursion}
\end{equation}
Thus the lost mass accumulates linearly across stages.

After fixing $V_t$, use an independent sample to estimate
$V_t^\top M_tV_t$, again with statistical Frobenius precision $\eta$.
Project that estimate onto the PSD cone to obtain $\widehat C$.
Nonexpansiveness of this projection in Frobenius norm, together with
the common support of dimension at most $k$, gives
\[
 E_t:=\norm{R-P_tM_tP_t}_1
 \leq\sqrt{k}\,(\eta+|c_{2t}|^{-1}\zeta).
\]
Summing \eqref{eq:merged-tt-recursion} and using
\eqref{eq:merged-hermite-coefficient} proves
\eqref{eq:merged-linear-leakage}.

The matrices $R_j$ are the tensor-train cores. Their total size is
polynomial in $t,d,k$, and contraction evaluates each required query
in polynomial time. The moment matrices have size at most $kd$,
and no computation requires an ambient $d^j$-dimensional matrix.
Distributing the failure probability over the stages, final estimate,
and matrix entries gives the stated bounds.
The algorithm replaces $R$ by zero whenever $\tr R>2$.
This replacement is unnecessary when $E_t<1$, since
$\tr R\leq\tr(P_tM_tP_t)+E_t\leq1+E_t$.
The replacement cannot increase $\Gamma$: when it occurs, the original
$E_t$ exceeds one, whereas the zero state's projection loss plus trace error
equals $\sum_i a_i\leq1$.
\end{proof}

\section{Robust direction covers from sequential Gaussian contractions}
\label{sec:robust-cover}

Fix $\theta=10^{-3}$ and the angular parameter $a=\theta/20$.
The trace bound in Lemma~\ref{lem:tt-compression} allows us to define
the following decoder. It returns null if the state has trace zero.
Otherwise, with probability $\tr R/2$, normalize $R$ to trace one;
with the remaining probability return null.

At each of the first $m=t-1$ tensor sites, let $S$ be the one-site
reduced state, sample $z\in\R^d$ with density
\begin{equation}
 \phi_d(z)z^\top Sz,
 \label{eq:merged-measurement-density}
\end{equation}
and normalize the contraction of the remaining state against $z$.
This density integrates to one. To sample it exactly, diagonalize
$S$, choose an eigenvector with probability its eigenvalue, use
independent standard Gaussians in the orthogonal directions, and on
that eigenvector use a random sign times the square root of a
$\chi_3^2$ variable. The latter variable is a sum of three independent
Gaussian squares. At the final site, let $Q$ be the remaining
trace-one matrix. Return a unit leading eigenvector if
\begin{equation}
 \lambda_{\max}(Q)\geq1-a^2/100;
 \label{eq:merged-decoder-test}
\end{equation}
otherwise return null.

The decoder has polynomial cost in $t,d,k$. To see this, factor
$\widehat C$ into at most $k$ columns and contract the corresponding
tensor-train vectors. Equivalently, $R$ is a matrix-product operator
with bond dimension at most $k^2$. Reduced states and successive
contractions require only polynomial-size matrices.

Let $\mu$ denote product standard Gaussian measure on the $m$
measurement vectors. For a measurement sequence $z=(z_1,\ldots,z_m)$,
write
$K_z=(z_1^\top\otimes\cdots\otimes z_m^\top)\otimes I_d$ and
$R_z=K_zRK_z^\top$. The successive normalization factors cancel:
before the final eigenvalue test, the decoder's joint measure and
last state are exactly
\begin{equation}
 d\Pr(z)=\tfrac12\tr(R_z)\,d\mu(z),\qquad
 Q_z=R_z/\tr(R_z).
 \label{eq:merged-joint-decoder-law}
\end{equation}
Each normalized state is considered only where its denominator is positive.
Moreover, $\int K_z^\top K_z\,d\mu(z)=I$.

\begin{lemma}[Stability of the sequential decoder]
\label{lem:decoder}
Let $P=P_t$ and let $R$ be the state constructed above. There are
universal constants $C,c>0$ such that, if
\begin{equation}
 t\geq C\log(Ck/\tau),\qquad
 \Gamma=L_t+\norm{R-PM_tP}_1\leq c a^4\tau,
 \label{eq:merged-decoder-regime}
\end{equation}
then, for every atom of weight $a_i\geq\tau$, the decoder returns, with
probability at least $a_i/3$, a non-null output within projective
distance $a$ of that atom. The probability of an output farther than $a$ from
every positive-weight atom is at most $\tau/100$.
\end{lemma}

\begin{proof}
First consider the ideal state $M_t$. Conditional on its latent
index $i$, the measurements are independent with density
$q_i(z)=\phi_d(z)(u_i^\top z)^2$. Put
\[
 c_i(z)=\prod_{r=1}^m u_i^\top z_r,\qquad
 w_i(z)=a_i c_i(z)^2,\qquad W(z)=\sum_i w_i(z),\qquad
 F_z=\sum_i w_i(z)u_iu_i^\top .
\]
Conditional on the latent index, the final state is $u_iu_i^\top$.
Without conditioning on that index, the final state is the posterior
mixture $F_z/W(z)$.

For unit $u,v$ and $r=|u^\top v|$, Gaussian integration gives
\begin{equation}
 \int\sqrt{q_uq_v}
 =\frac2\pi\left(\sqrt{1-r^2}+r\arcsin r\right)
 \leq1-\frac{1-r^2}{4}\leq1-\frac{D(u,v)^2}{8}.
 \label{eq:merged-affinity}
\end{equation}
For the first inequality, subtract the middle expression from
$3/4+r^2/4$. The difference is zero at $r=1$ and has derivative
$r/2-(2/\pi)\arcsin r\leq0$. For a cutoff $s>0$,
\[
 \sum_{D(u_i,u_j)>s}\int\frac{w_iw_j}{W}\,d\mu
 \leq\sum_{D(u_i,u_j)>s}\int\sqrt{w_iw_j}\,d\mu
 \leq\left(\sum_i\sqrt{a_i}\right)^2 e^{-ms^2/8}
 \leq k e^{-ms^2/8}.
\]
Integrands with $W=0$ are defined to be zero. Thus the total
unnormalized mass of pairs consisting of a latent index and a measurement
outcome, for which the posterior mass on directions farther than $s$
from the latent atom exceeds $b$, is
at most $k b^{-1}e^{-ms^2/8}$.

Set
\begin{equation}
 s=a/100,\qquad b=a^2/10^4.
 \label{eq:merged-posterior-margins}
\end{equation}
Since $a$ is fixed, the first condition in
\eqref{eq:merged-decoder-regime}, with a sufficiently large universal
$C$, bounds this exceptional mass by $\tau/100$. On every
remaining pair with latent index $i$,
\begin{equation}
 u_i^\top(F_z/W)u_i\geq1-s^2-b=1-a^2/5000.
 \label{eq:merged-ideal-margin}
\end{equation}
This leaves a strict margin in the acceptance test.

We next compare the estimated and ideal states. Define
\[
 v_i(z)=K_zP u_i^{\otimes t},\quad
 C_z=\sum_i a_i v_i(z)v_i(z)^\top,\quad
 e_z=\sum_i a_i\norm{v_i(z)-c_i(z)u_i}^2,\quad
 D_z=\norm{R_z-C_z}_1 .
\]
The completeness identity for Gaussian contractions gives
\begin{equation}
 \int e_z\,d\mu=L_t,\qquad
 \int D_z\,d\mu\leq E_t:=\norm{R-PM_tP}_1.
 \label{eq:merged-instrument-bounds}
\end{equation}
For the second inequality, apply $K_z$ to the positive and negative
parts of $R-PM_tP$ separately. Integration preserves their traces,
and the triangle inequality proves the assertion.

With column matrices $A_i=\sqrt{a_i}c_i u_i$ and
$B_i=\sqrt{a_i}v_i$, the trace-norm product inequality yields
\begin{equation}
 \norm{C_z-F_z}_1\leq2\sqrt{We_z}+e_z,\qquad
 \tr R_z\leq2W+2e_z+D_z .
 \label{eq:merged-pointwise-perturbation}
\end{equation}
Call $z$ good when
\[
 e_z\leq10^{-12}a^4W,\qquad
 D_z\leq10^{-6}a^2W.
\]
On this set, for $W>0$,
\[
 \norm{R_z-F_z}_1\leq a^2W/10^5,\qquad
 \tr R_z\geq0.99W,\qquad
 \left\lVert\frac{R_z}{\tr R_z}-\frac{F_z}{W}\right\rVert_1
 \leq a^2/10^4.
\]
Indeed, the unnormalized bound follows from
\eqref{eq:merged-pointwise-perturbation}; normalizing increases it
by a factor at most $2/(1-a^2/10^5)$. When $W=0$, the good-event conditions imply
$R_z=0$, so there is no output.

By \eqref{eq:merged-instrument-bounds}, the bad set has ideal mass
at most
$10^{12}a^{-4}L_t+10^6a^{-2}E_t\leq10^{12}a^{-4}\Gamma$.
The second inequality of \eqref{eq:merged-pointwise-perturbation}
then bounds its unnormalized mass under the estimated state by
$3\cdot10^{12}a^{-4}\Gamma$.
In particular, taking $c\leq10^{-16}$ in
\eqref{eq:merged-decoder-regime} makes each of these masses at most
$\tau/100$.

On a good accepted outcome, the returned vector $v$ satisfies
\[
 v^\top(F_z/W)v\geq1-a^2/100-a^2/10^4\geq1-a^2/50.
\]
This is a positive weighted average of $(v^\top u_i)^2$. At least
one positive-weight atom therefore has $D(v,u_i)^2\leq a^2/25<a^2$.
Thus all outputs far from every atom arise on the bad set, proving
the soundness probability bound by
\eqref{eq:merged-joint-decoder-law}.

For a good pair with latent index $i$ satisfying
\eqref{eq:merged-ideal-margin}, the Rayleigh quotient for the estimated state at
$u_i$ is at least $1-3a^2/10^4$. The test
\eqref{eq:merged-decoder-test} therefore passes. Its leading vector
$v$ obeys $v^\top(F_z/W)v\geq1-4a^2/10^4$. For any PSD trace-one
matrix $Q$,
\[
 u^\top Qu+v^\top Qv
 \leq\lambda_{\max}(uu^\top+vv^\top)=1+|u^\top v|.
\]
Using $Q=F_z/W$ shows
$D(u_i,v)^2\leq12a^2/10^4<a^2/100$.

For an atom with $a_i\geq\tau$, excluding the posterior exception
and the bad set removes at most $2\tau/100$ from its latent mass.
On the remaining measurement sequences the output is within $a$ of
$u_i$, and its density is at least $0.99W/2\geq0.99w_i/2$.
Its output probability is consequently at least
$0.99(a_i-2\tau/100)/2\geq a_i/3$.
\end{proof}

\subsection{Frequency estimates and pruning}

\begin{proposition}[Robust coarse cover]
\label{prop:robust-cover}
Fix $\theta=10^{-3}$. There are universal constants $C,C_*>0$ such
that, for $\tau\in(0,1]$, $\omega\in(0,1)$, and $\eta>0$, the conditions
\begin{equation}
 t\geq C\log(Ck/\tau),\qquad
 \tau\geq C_*\sqrt{k}(t+1)^{9/4}(\zeta+\eta)
 \label{eq:merged-cover-interface}
\end{equation}
ensure that a randomized procedure returns, with probability at least
$1-\omega$, a set $W\subseteq\sphere$ of at most $k$ centers such that
every atom with $a_i\geq\tau$ is within projective distance $\theta$
of $W$, and every center is within projective distance $\theta$ of
some positive-weight atom. The sample and operation counts are bounded by
\begin{equation}
 B^{O(1)}\exp(O(t))
 \poly(d,k,\tau^{-1},\eta^{-1},\log(2/\omega)).
 \label{eq:merged-robust-cover-cost}
\end{equation}
In the conditional-mean model, $\zeta=0$ and the factor $B^{O(1)}$
can be replaced by $1+\sigma^2$. Thus choosing
$t=\lceil C\log(Ck/\tau)\rceil$ and a sufficiently small
inverse-polynomial $\eta$ gives a coarse cover with polynomial cost in
the stated parameters. The sample and operation bounds and the limit
of $k$ output centers hold on every outcome.
\end{proposition}

\begin{proof}
Construct the state with Lemma~\ref{lem:tt-compression}. Choose
$C_*$ large enough that \eqref{eq:merged-linear-leakage} implies
the decoder condition \eqref{eq:merged-decoder-regime}.
Conditional on successful construction, draw
\[
 N=\left\lceil C\tau^{-1}\log(2k/\omega)\right\rceil
\]
independent decoder outputs. A union bound ensures, with the required
success probability, that each heavy atom has a non-null candidate
within distance $a$.
Draw a single independent validation sample of
$\lceil C\tau^{-2}\log(2N/\omega)\rceil$ decoder outputs. Conditional
on the candidate list, Hoeffding's inequality and a union bound
estimate the mass of every candidate's projective $2a$-ball to
accuracy $\tau/100$. Null outcomes belong to none of these balls.
Retain candidates with estimated mass at least $\tau/6$.

A candidate within $a$ of a heavy atom has a $2a$-ball containing
that atom's entire $a$-ball. Its mass is at least $\tau/3$, so it is
retained. Conversely, a candidate farther than $3a$ from every atom
has its whole $2a$-ball outside all atom neighborhoods. Its mass is at most
$\tau/100$, so it is rejected. Thus every retained candidate has a
positive-weight witness atom within $3a$.

Greedily select a maximal subset of retained candidates with mutual
projective distances greater than $6a$. Two selected candidates
cannot share a witness, so the selected set has at most $k$ points.
Maximality puts every retained candidate within $6a$ of a selected
one. Each heavy atom is therefore within $7a<\theta$ of the selected
set, while every selected point is within $3a<\theta$ of an atom.

Distribute the failure probability among moment estimation, candidate
coverage, and validation. Decoding and distance computations have
polynomial cost and need no further labeled samples. This proves
\eqref{eq:merged-robust-cover-cost} and its conditional-mean variant.
The algorithm obeys the prescribed sample bounds, replaces states of
trace greater than two by zero, and returns the empty set if the final
greedy set exceeds $k$. The final condition is never triggered on the
success event. Thus the bounds hold for every realization of the
algorithm's randomness.
\end{proof}

%% file: localization.tex
\section{Exact Gaussian localization and geometric refinement}
\label{sec:localization}

Fix the universal angular constant $\theta=10^{-3}$.
For a finite list $L$ of unit vectors and a radius $s>0$,
let $\Prune(L,s)$ be a greedy maximal subset whose distinct points have
projective distance greater than $2s$. Scan the list and retain a point
exactly when its distance from every previously retained
point exceeds $2s$.

\begin{lemma}[Pruning a sound list]
\label{lem:pruning}
Suppose every point of $L$ is within projective distance $s$ of some atom
in a set of at most $k$ atoms, and every designated heavy atom is within
$s$ of a point of $L$. Then $W=\Prune(L,s)$ has at most $k$ points,
covers every heavy atom to radius $3s$, and every point of $W$ is
within $s$ of an atom.
\end{lemma}

\begin{proof}
Assign a witness atom to each retained point. Two retained points cannot
share a witness, because their mutual distance would then be at most
$2s$. Hence $|W|\leq k$. Maximality means that every point of $L$ is
within $2s$ of a retained point. The triangle inequality therefore gives
a retained point within $3s$ of each heavy atom. Soundness is preserved
because $W\subseteq L$.
\end{proof}

\subsection{An exact change of Gaussian measure}

Fix $w\in\sphere$ and $\rho\in(0,1]$, and write
\begin{equation}
 T=T_{w,\rho}=\rho ww^\top+I-ww^\top,\qquad
 A_{w,\rho}(x)=
 \exp\!\left[-\frac{(\rho^{-2}-1)(w^\top x)^2}{2}\right].
 \label{cm:local-map}
\end{equation}
Draw an independent example $(X,Y)$ from the original distribution.
Conditional on $X$, accept it with probability $A_{w,\rho}(X)$ using
randomness independent of the label.
This is the Gaussian soft-localization identity used in
\citet[Lemma~14]{diakonikolas2021agnostic}.

\begin{lemma}[Exact Gaussian localization]
\label{lem:localization}
Suppose $X\sim\N(0,I_d)$, $\E[Y^2]\leq B$, and
$g(x)=\E[Y\mid X=x]$. Let
\[
 f(x)=\sum_i a_i\relu(u_i^\top x),\qquad
 a_i\geq0,\quad \sum_i a_i\leq1,\quad u_i\in\sphere,
\]
be a comparator with $\norm{g-f}_{L^2}\leq\zeta$.
The acceptance probability in \eqref{cm:local-map} is $\rho$.
Conditional on acceptance, $X\sim\N(0,T^2)$ and
$Z=T^{-1}X\sim\N(0,I_d)$. In these coordinates the label mean is $g(TZ)$,
and the transformed comparator is
\begin{equation}
 f(Tz)=\sum_i a_i\norm{Tu_i}_2
       \relu\!\left(\frac{(Tu_i)^\top z}{\norm{Tu_i}_2}\right).
 \label{cm:transformed-comparator}
\end{equation}
Its total weight is at most one, and each individual weight is at least
$a_i\rho$. Moreover,
\begin{equation}
 \norm{g(TZ)-f(TZ)}_{L^2(Z)}\leq\frac{\zeta}{\sqrt\rho},
 \qquad
 \E[Y^2\mid\text{accepted}]\leq\frac B\rho.
 \label{cm:localized-label-bounds}
\end{equation}
If $g=f$ and $\E[(Y-f(X))^2]\leq\sigma^2$, the accepted model additionally
satisfies
\begin{equation}
 \E[Y\mid Z,\text{accepted}]=f(TZ),\qquad
 \E[(Y-f(TZ))^2\mid\text{accepted}]\leq\frac{\sigma^2}{\rho}.
 \label{cm:localized-conditional-noise}
\end{equation}
\end{lemma}

\begin{proof}
Multiplying the standard Gaussian density by $A_{w,\rho}(x)$
changes its quadratic form to
\[
 x^\top[I+(\rho^{-2}-1)ww^\top]x=x^\top T^{-2}x.
\]
Since $\det T=\rho$, the resulting unnormalized density integrates to
$\rho$, and its normalization is the density of $\N(0,T^2)$. The coordinate
change $Z=T^{-1}X$ thus gives a standard Gaussian.

Acceptance depends only on $X$ and independent auxiliary randomness, so
it preserves the conditional label law at each input. Consequently the
accepted label mean is $g(TZ)$. Positive homogeneity of ReLU yields
\eqref{cm:transformed-comparator}. The singular values of $T$ lie in
$[\rho,1]$, and hence
$\rho\leq\norm{Tu_i}_2\leq1$, proving the weight assertions.

For every nonnegative integrable random variable $U=U(X,Y)$,
\begin{equation}
 \E[U\mid\text{accepted}]
 =\frac{\E[U A_{w,\rho}(X)]}{\rho}
 \leq\frac{\E U}{\rho}.
 \label{cm:acceptance-moment-bound}
\end{equation}
Apply this first to $U=(g(X)-f(X))^2$ and then to $U=Y^2$ to obtain
\eqref{cm:localized-label-bounds}. In the conditional-mean model,
apply it also to $U=(Y-f(X))^2$. Together with the conditional-mean
identity above, this proves \eqref{cm:localized-conditional-noise}.
\end{proof}

The acceptance test can be implemented using only Gaussian sampling,
arithmetic and comparisons. If $E\sim\chi^2_2$ is independent of the
example, accept when
$E\geq(\rho^{-2}-1)(w^\top X)^2$; its survival probability is exactly
\eqref{cm:local-map}. The variable $E$ is the sum of squares of two
independent standard Gaussians.
When a uniform random variable is needed, two independent
$\chi^2_2$ variables $E_1,E_2$ give
$E_1/(E_1+E_2)\sim\operatorname{Unif}[0,1]$.

To obtain up to $M$ accepted examples, limit the number of
proposals from the original distribution to
\begin{equation}
 \left\lceil\frac{8(M+\log(1/\omega))}{\rho}\right\rceil.
 \label{cm:rejection-budget}
\end{equation}
The number accepted among this many independent proposals has a binomial
distribution with mean at least
$8(M+\log(1/\omega))$. A Chernoff lower-tail bound therefore makes the
probability of obtaining fewer than $M$ examples at most $\omega$.
Couple this procedure to the infinite sequence of accepted examples,
which is i.i.d.\ from the transformed model. Whenever the procedure
obtains the required number of examples, they agree with the initial
segment of that sequence. Estimation bounds therefore apply to the
i.i.d.\ sequence, with an additional failure event for insufficient
acceptances.

Every center and every round uses fresh samples from the original
distribution. Thus $\rho^{-1}$ is the rejection-sampling overhead at the
current scale. This overhead and the factors in
\eqref{cm:localized-label-bounds} and
\eqref{cm:localized-conditional-noise} depend only on that scale and do
not accumulate along a refinement path.

\subsection{Contraction of the inverse map on a cap}

\begin{lemma}[Local inverse contraction]
\label{lem:inverse}
Suppose $D(u,w)\leq\rho\leq1/10$, and put
$v_u=Tu/\norm{Tu}_2$. Then $|v_u^\top w|>0.70$. Every projective
$\theta$-approximation $v$ to $v_u$ therefore passes the filter
$|v^\top w|\geq2/3$.

Conversely, suppose $v$ passes this filter and
$v_i=Tu_i/\norm{Tu_i}_2$ satisfies $D(v,v_i)\leq\theta$. Then
\begin{equation}
 D\!\left(\frac{T^{-1}v}{\norm{T^{-1}v}_2},u_i\right)
 \leq3\rho\theta.
 \label{cm:inverse-bound}
\end{equation}
\end{lemma}

\begin{proof}
Orient $u$ so that $u^\top w\geq0$. The distance assumption gives
$u^\top w\geq1-\rho^2/2\geq0.995$ and
$\norm{u-(u^\top w)w}_2\leq\rho$. Hence
\[
 v_u^\top w
 =\frac{\rho(u^\top w)}
 {\sqrt{\rho^2(u^\top w)^2+\norm{u-(u^\top w)w}_2^2}}
 \geq\frac{0.995}{\sqrt2}>0.70.
\]
A projective $\theta$-neighbor has absolute inner product with $w$
at least $0.70-\theta=0.699>2/3$.

For the second assertion orient $v$ toward $w$, and choose the sign of its
witness so that $\norm{v-v_i}_2\leq\theta$. Write
$a=v^\top w$ and $b=v_i^\top w$. Then
$a\geq2/3$ and $b\geq2/3-\theta>0$. With
$v_\perp=v-(v^\top w)w$, the inverse direction map is
\begin{equation}
 F_\rho(v)=\frac{T^{-1}v}{\norm{T^{-1}v}_2}
 =\frac{w+\rho v_\perp/a}
 {\sqrt{1+\rho^2\norm{v_\perp/a}_2^2}}.
 \label{cm:inverse-map}
\end{equation}
Let $c=v^\top v_i$. The nonnegative determinant of the Gram matrix of
$v,v_i,w$ gives $a^2+b^2-2abc\leq1-c^2$, and thus
\[
 \norm{bv-av_i}_2^2\leq1-c^2
 \leq2(1-c)=\norm{v-v_i}_2^2.
\]
Since $bv-av_i=bv_\perp-a(v_i)_\perp$, it follows that
\[
 \left\lVert\frac{v_\perp}{a}
 -\frac{(v_i)_\perp}{b}\right\rVert_2
 \leq\frac{\norm{v-v_i}_2}{ab}.
\]
The map $G(x)=(w+x)/\sqrt{1+\norm{x}_2^2}$, for $x\perp w$, is
$1$-Lipschitz: for $h\perp w$ its derivative obeys
\[
 \norm{DG(x)h}_2^2
 =\frac{\norm{h}_2^2}{1+\norm{x}_2^2}
 -\frac{\inner{x}{h}^2}{(1+\norm{x}_2^2)^2}
 \leq\norm{h}_2^2.
\]
Apply this to the slope coordinates in \eqref{cm:inverse-map}. Because
$1/[(2/3)(2/3-\theta)]<3$ for $\theta=10^{-3}$,
\[
 \norm{F_\rho(v)-F_\rho(v_i)}_2
 \leq\frac{\rho\theta}{ab}<3\rho\theta.
\]
The vector $F_\rho(v_i)$ is an orientation of $u_i$, proving the result.
\end{proof}

%% file: conditional.tex
\section{Learning in the conditional-mean model}
\label{sec:conditional}

We prove Theorem~\ref{thm:conditional} under the conditional-mean model
\begin{equation}
 \begin{gathered}
 X\sim\N(0,I_d),\qquad \E[Y\mid X]=f(X),\qquad
 \E[(Y-f(X))^2]\leq\sigma^2,\\
 f(x)=\sum_{i=1}^{k_0}a_i\relu(u_i^\top x),\qquad
 k_0\leq k,\quad a_i\geq0,\quad
 \sum_i a_i\leq1,\quad u_i\in\sphere.
 \end{gathered}
 \label{cm:conditional-model}
\end{equation}
Zero-weight atoms are omitted. The conditional variance of the labels
given $X$ may be unbounded as a function of $X$.
Positivity and the mass constraint imply
$\E[f(X)^2]\leq1/2$, and hence $\E[Y^2]\leq1/2+\sigma^2$.

\subsection{Direction covers in the conditional-mean model}

\begin{corollary}[Direction covers with linear noise dependence]
\label{cm:cover}
For $\tau,\omega\in(0,1)$, the cover procedure of
Proposition~\ref{prop:robust-cover}, applied to
\eqref{cm:conditional-model}, returns a list of at most $k$ unit vectors.
With probability at least $1-\omega$, every atom of weight at least
$\tau$ lies within projective distance $\theta$ of the list, and every
listed point lies within $\theta$ of a positive-weight atom.
The numbers of samples and real-arithmetic operations are at most
\[
 (1+\sigma^2)\poly(d,k,\tau^{-1},\log(2/\omega)).
\]
The size bound holds on every outcome.
\end{corollary}

\begin{proof}
The population moment perturbation in Proposition~\ref{prop:robust-cover}
is zero. Choose
\[
 t=\left\lceil C\log(Ck/\tau)\right\rceil,\qquad
 \eta=\frac{c\tau}{\sqrt{k}(t+1)^{9/4}},
\]
where $C$ is sufficiently large and $c$ sufficiently small to satisfy
the conditions of that proposition.
Use the conditional-mean estimator in Lemma~\ref{lem:query-access},
with the signal bounds from Lemma~\ref{lem:moment-signal}.
The required multiplier norms, clipping levels, matrix dimensions and
inverse precisions depend on $d,k,\tau,\omega$, but not on $\sigma$.
The conditional-mean estimator has sample cost linear in
$1+\sigma^2$, and the computations on each example have the same
noise-independent bounds. All subsequent tensor contractions, decoder
draws and validation calculations use the already estimated state and
have polynomial cost independent of $\sigma$.
Since $\exp(O(t))$ and $\eta^{-1}$ are polynomial in $k,\tau^{-1}$,
the stated complexity follows. The coverage, soundness and size
guarantees follow from Proposition~\ref{prop:robust-cover}.
\end{proof}

Denote this specialization by $\Coarse_0(\tau)$, with the current
noise second-moment bound and an assigned failure probability on each
call. Its soundness witnesses need not be heavy.

\subsection{The refinement algorithm}

For target squared prediction error $\eps\in(0,1/10)$, set
\begin{equation}
 \theta=10^{-3},\qquad
 a_0=\frac{\sqrt\eps}{100k},\qquad
 \rho_*=\frac{\sqrt\eps}{100},\qquad
 \rho_0=3\theta.
 \label{cm:parameters}
\end{equation}
Call an original atom heavy if $a_i\geq a_0$.
Proceed as follows.
\begin{enumerate}
\item Run $\Coarse_0(a_0)$ on original examples, and set
      $W=\Prune(L,\theta)$ and $\rho=\rho_0$.
\item If $\rho\leq\rho_*$, proceed to step 4. Otherwise, perform the
      following operation for each $w\in W$.
      Use the sampler of Lemma~\ref{lem:localization} with
      $T_{w,\rho}$ on fresh original examples. In the accepted
      coordinates $Z=T_{w,\rho}^{-1}X$, run $\Coarse_0(a_0\rho)$ with
      noise bound $\sigma^2/\rho$. Retain only outputs satisfying
      $|v^\top w|\geq2/3$ and map each to
      $T_{w,\rho}^{-1}v/\norm{T_{w,\rho}^{-1}v}_2$.
\item Merge these mapped lists into $L'$, put $s=3\theta\rho$, replace
      $W$ by $\Prune(L',s)$, and replace $\rho$ by $9\theta\rho$.
      Return to step 2.
\item Include both ReLU orientations of the final centers, fit
      nonnegative coefficients with sum at most one, and perform the
      exact sparsification below.
\end{enumerate}
If $W$ is empty, refinement may terminate immediately and the output is
the zero predictor.

Every adaptive query is fixed before independent samples are drawn to
estimate it. The sample bounds for moment estimation, decoder sampling,
validation and rejection sampling are fixed in advance from the problem
parameters. The compression rank and the size of every cover or pruned
list are restricted to $k$; if a sampling bound is exceeded or a rank or
list would exceed its allowed size, the algorithm terminates and returns
the zero network. The coefficient fit remains in its feasible simplex
at every iteration, and every completed fit undergoes exact
sparsification. Thus the output satisfies the sign, mass and width
constraints on every outcome, including estimation failures.

\subsection{The refinement invariant and direction recovery}

Initially Corollary~\ref{cm:cover} and Lemma~\ref{lem:pruning} give at most
$k$ centers, completeness radius $\rho_0=3\theta$, and soundness radius
$\rho_0/3=\theta$. Suppose a current list has at most $k$ centers,
covers every original heavy atom to radius $\rho\leq1/10$, and is
sound to radius $\rho/3$.

An original heavy atom is within $\rho$ of at least one current center.
In that center's localized model, its weight is at least $a_0\rho$ by
Lemma~\ref{lem:localization}. Corollary~\ref{cm:cover} supplies a
projective $\theta$-neighbor of its transformed direction.
Lemma~\ref{lem:inverse} ensures that this neighbor passes the filter and
maps back within $s=3\theta\rho$ of the original atom. Conversely,
every retained coarse output has some transformed witness atom.
The same inverse lemma maps it within $s$ of an original atom, even if
that witness is light.

The merged list is therefore complete for heavy atoms and sound for
all outputs at radius $s$. Pruning leaves at most $k$ centers and
establishes the new radii
\begin{equation}
 \rho_{\rm new}=3s=9\theta\rho=0.009\rho,\qquad
 \text{soundness radius }s=\rho_{\rm new}/3.
 \label{cm:contraction}
\end{equation}
This proves the invariant inductively and gives
$O(\log(1/\rho_*))$ rounds. If an empty list occurs on the simultaneous
success event, there are no heavy atoms, so the zero comparison
predictor below applies. All light atoms remain in the moment
calculations and transformed models; they are omitted only when
constructing the comparison predictor for the error bound.

\begin{proposition}[Parameterized recovery of projective directions]
\label{cm:directions}
Under \eqref{cm:conditional-model}, for every weight threshold
$\tau\in(0,1)$, radius $r\in(0,1/10)$ and failure probability
$\delta\in(0,1/10)$, the cover-and-refinement procedure returns a set
$W$ of at most $k$ unit vectors such that, with probability at least
$1-\delta$:
\begin{enumerate}
\item every atom with $a_i\geq\tau$ is within projective distance $r$
      of $W$;
\item every point of $W$ is within projective distance $r/3$ of some
      positive-weight atom.
\end{enumerate}
The numbers of samples and real-arithmetic operations are at most
\[
 (1+\sigma^2)\poly(d,k,\tau^{-1},r^{-1},\log(1/\delta)).
\]
\end{proposition}

\begin{proof}
Use the first three algorithm steps with $a_0=\tau$ and $\rho_*=r$.
The initial list has completeness radius $\rho_0$ and soundness radius
$\rho_0/3$. Every refinement preserves this ratio by
\eqref{cm:contraction}, so stopping at $\rho\leq r$ gives both
conclusions. If the initial radius already satisfies $\rho_0\leq r$,
the same conclusion holds without a localized call.

Before any localized call, $\rho>r$. Its weight threshold is
$\tau\rho>\tau r$, its noise bound is
$\sigma^2/\rho<\sigma^2/r$, and its rejection overhead is
$\rho^{-1}<r^{-1}$. Corollary~\ref{cm:cover} and
\eqref{cm:rejection-budget} therefore bound each call by the claimed
polynomial. There are $O(\log(1/r))$ rounds and at most $k$ calls per
round. Allocate the total failure probability among these calls and
their rejection-sampling stages, condition on previously constructed
objects,
and apply their guarantees to fresh samples. A union bound proves the
probability statement. The prediction analysis below uses the same
sample bounds and independence argument.
\end{proof}

\subsection{Nonnegative regression and exact sparsification}
\label{cm:regression}

Let $v_1,\ldots,v_s$, with $s\leq k$, be the final centers. Include both
orientations and denote the resulting $m=2s\leq2k$ features by
$\phi_j(x)=\relu(\widetilde v_j^\top x)$, where the
$\widetilde v_j$ range over $v_1,-v_1,\ldots,v_s,-v_s$. Fit the
coefficients over the simplex
\begin{equation}
 \coeffset=\left\{c\in\R^m:c_j\geq0,\ \sum_jc_j\leq1\right\}.
 \label{cm:simplex}
\end{equation}
If $s=0$, take the zero predictor.

\paragraph{A comparison network.}
For each heavy atom choose a closest correctly oriented feature, assign
its true weight to that feature, and add weights when several atoms
choose the same feature. Omit light atoms. The resulting
$h_{\rm comp}$ has coefficients in \eqref{cm:simplex}.

For real $a,b$,
\[
 [\relu(a)-\relu(b)]^2+
 [\relu(-a)-\relu(-b)]^2\leq(a-b)^2.
\]
Pairing $X$ with $-X$ and using $\E[XX^\top]=I_d$ gives
\begin{equation}
 \norm{\relu(u^\top X)-\relu(v^\top X)}_{L^2}
 \leq\frac{\norm{u-v}_2}{\sqrt2}.
 \label{cm:relu-lipschitz}
\end{equation}
Also $\norm{\relu(u^\top X)}_{L^2}=1/\sqrt2$ for unit $u$. Since
there are at most $k$ omitted atoms, each of weight below $a_0$,
the final cover gives
\begin{equation}
 \norm{f-h_{\rm comp}}_{L^2}
 \leq\frac{\rho_*+ka_0}{\sqrt2}
 <\frac{\sqrt\eps}{40}.
 \label{cm:comparison}
\end{equation}
If the output list is empty on the success event, all atoms are light
and the same inequality holds with $h_{\rm comp}=0$.

\paragraph{Estimating the regression coefficients.}
For two unit directions of inner product $r$, the Gaussian ReLU kernel is
\begin{equation}
 K(r)=\frac{\sqrt{1-r^2}+(\pi-\arccos r)r}{2\pi}.
 \label{cm:relu-kernel}
\end{equation}
Indeed, integrate the product of the positive parts of the two Gaussian
projections in polar coordinates in their two-dimensional span; the radial integral
and the angular interval on which both projections are positive give
the displayed expression. The endpoint values follow by continuity.
Thus the feature Gram matrix is exactly
$G_{ab}=K(\widetilde v_a^\top\widetilde v_b)$.

Set $b_j=\E[Y\phi_j(X)]$. Use fresh original samples and clipped
feature multipliers to obtain
\begin{equation}
 \norm{\widehat b-b}_\infty\leq\eps/100.
 \label{cm:correlations}
\end{equation}
The following estimator applies under \eqref{cm:conditional-model}.
For a clipping level $M>0$, put
$\phi_{j,M}=\min\{\phi_j,M\}$. Since
$\E[\relu(u^\top X)^4]=3/2$, positivity and Cauchy--Schwarz give
\[
 \E[f(X)^2\phi_j(X)^2]\leq3/2,\qquad
 \E[\phi_j(X)^2]=1/2.
\]
Writing $Y=f(X)+\xi$, the conditional mean of $\xi$ is zero, so clipping
introduces no noise bias. Consequently,
\begin{align}
 |\E[Y\phi_{j,M}(X)]-b_j|
 &\leq
 \sqrt{\E[f(X)^2\phi_j(X)^2]}\,
 \sqrt{\Prob(\phi_j(X)>M)}
 \leq\frac{\sqrt3}{2M},\label{cm:correlation-bias}\\
 \E[Y^2\phi_{j,M}(X)^2]
 &\leq M^2(1/2+\sigma^2).\label{cm:correlation-variance}
\end{align}
Choose $M=100\sqrt3/\eps$. Estimating each clipped correlation to
statistical error $\eps/200$ leaves at most $\eps/200$ clipping bias.
Median-of-means estimation uses
$O((1+\sigma^2)\eps^{-4}\log(2m/\omega))$ samples per feature, with
simultaneous failure probability at most $\omega$.

Minimize
\begin{equation}
 \widehat Q(c)=c^\top Gc-2\widehat b^\top c,\qquad c\in\coeffset,
 \label{cm:quadratic}
\end{equation}
to additive optimization error $\eps/100$. Projected gradient descent
suffices: $G\succeq0$,
$\norm{G}_{\op}\leq\tr G=m/2$, and $\coeffset$ has bounded Euclidean
diameter. The standard smooth convex descent bound therefore gives
$\poly(k,\eps^{-1})$ real operations, including projection onto the
simplex. The argument requires no lower bound on the smallest eigenvalue
of $G$.

Let $Q(c)=c^\top Gc-2b^\top c$. The conditional-mean identity gives
\[
 \norm{\textstyle\sum_jc_j\phi_j-f}_{L^2}^2
 =Q(c)+\norm{f}_{L^2}^2.
\]
On the feasible simplex, \eqref{cm:correlations} implies
$|\widehat Q(c)-Q(c)|\leq2\eps/100$. Comparison with the coefficients
of $h_{\rm comp}$, including the optimization error, yields
\begin{equation}
 \norm{h-f}_{L^2}^2
 \leq\norm{h_{\rm comp}-f}_{L^2}^2+\frac{5\eps}{100}
 <\frac{\eps}{1600}+\frac{5\eps}{100}<\eps.
 \label{cm:final-risk}
\end{equation}

\begin{proposition}[Width reduction without prediction error]
\label{cm:sparsify}
Any nonnegative combination of the $2s$ features
$\relu(v_j^\top x),\relu(-v_j^\top x)$ can be represented with at most
$s+\operatorname{rank}(V)$ nonzero ReLU units, where
$V=[v_1\,\cdots\,v_s]$. This representation can be computed in polynomially
many real operations and preserves the represented function and the
total coefficient mass exactly.
\end{proposition}

\begin{proof}
Write the coefficients in pair $j$ as $c_j^+,c_j^-\geq0$, and put
$t_j=c_j^++c_j^-$ and $z_j=c_j^+-c_j^-$. Then $|z_j|\leq t_j$ and
\begin{equation}
 h(x)=\frac12\sum_{j=1}^s t_j|v_j^\top x|
       +\frac12(Vz)^\top x.
 \label{cm:even-odd-sparsification}
\end{equation}
Keep $t$ and $Vz$ fixed. Let $I=\{j:|z_j|<t_j\}$ be the unsaturated
coordinates. If $|I|>\operatorname{rank}(V)$, compute a nonzero vector
$a$ supported on $I$ with $Va=0$. Move from $z$ along $z+\lambda a$
until the first boundary $|z_j|=t_j$ is reached. A nonzero feasible step
exists because the moving coordinates are strictly interior to their
intervals. Previously saturated coordinates do not move. Thus each
step reduces $|I|$ by at least one, and after at most $s$ steps there
are at most $\operatorname{rank}(V)$ unsaturated coordinates.

Recover $c_j^\pm=(t_j\pm z_j)/2$. A saturated positive-mass coordinate
uses one unit, an unsaturated coordinate uses two, and a zero-mass
coordinate uses none. The total is at most
$s+\operatorname{rank}(V)$. Equation~\eqref{cm:even-odd-sparsification}
and $\sum_jt_j$ are unchanged. Nullspace computation and movement to
the first box boundary take polynomially many real operations.
\end{proof}

Since $s\leq k$, applying Proposition~\ref{cm:sparsify} gives output
width at most $k+\min\{k,d\}$ without changing the prediction error or
total coefficient mass. Evaluation of the resulting network takes
$O(dk)$ operations.

\subsection{Complexity and success probability}
\label{cm:complexity}

The maximum number of refinement rounds is
\begin{equation}
 L=\max\left\{0,
 \left\lceil\frac{\log(3\theta/\rho_*)}
 {\log(1/(9\theta))}\right\rceil\right\}
 =O(\log(1/\eps)).
 \label{cm:rounds}
\end{equation}
There is one initial cover call and at most $kL$ localized calls, at
most $kL$ rejection-sampling stages, and one final correlation-estimation
stage. Allocate failure probability at most
$\delta/[4(2+2kL)]$ to each cover, each rejection-sampling stage and the final
estimation stage. Within a cover, distribute its assigned probability
among the moment estimates, decoder samples and validation
events in Proposition~\ref{prop:robust-cover}. The sum is below
$\delta$.

For each adaptive call, condition on everything constructed before the
call and use fresh independent data and auxiliary randomness.
The corresponding conditional failure guarantee is uniform over the
previously fixed queries and transformations. Summing these
conditional bounds proves the unconditional union bound. On this
simultaneous success event, the refinement invariant also shows that
the list-size limit cannot cause early termination.

Before a localized call, $\rho>\rho_*$, and consequently
\begin{equation}
 \tau_{\rm call}=a_0\rho>a_0\rho_*=\frac{\eps}{10^4k},
 \qquad
 \rho^{-1}<100\eps^{-1/2}.
 \label{cm:worst-scales}
\end{equation}
The initial threshold $a_0$ is also inverse-polynomial in $k,\eps^{-1}$.
The cover degree is $O(\log(k/\tau_{\rm call}))$, so its exponential
degree factor and all required inverse statistical precisions are
polynomial in $k,\eps^{-1}$. By Corollary~\ref{cm:cover}, the number of
labeled examples required for each estimate is linear in $1+\sigma^2$
on the original distribution, or in $1+\sigma^2/\rho$ after localization.
The proposal bound \eqref{cm:rejection-budget} adds a
factor $\rho^{-1}$. Equation~\eqref{cm:worst-scales} absorbs both
localization factors into a polynomial in $\eps^{-1}$.

The total cost is the sum of the costs of the at most $1+kL$ cover
calls. Each call draws fresh examples directly from the original
distribution, so the rejection overhead is incurred only once per call.
Median-of-means estimation and the failure allocation contribute
polynomial dependence on $\log(1/\delta)$, with logarithms in the
other parameters absorbed into their polynomial bound. Final fitting,
kernel evaluation and exact sparsification have polynomial cost.
The total numbers of samples and real-arithmetic operations are therefore
\[
 (1+\sigma^2)\poly(d,k,\eps^{-1},\log(1/\delta)).
\]
Equations~\eqref{cm:comparison} and \eqref{cm:final-risk} prove the
prediction guarantee, and Proposition~\ref{cm:sparsify} proves the
width guarantee. This completes the proof of
Theorem~\ref{thm:conditional}.

\begin{corollary}[Confidence amplification with logarithmic sample overhead]
\label{cm:confidence}
For a fixed universal polynomial $P$, the sample count in
Theorem~\ref{thm:conditional} can be bounded by
\[
 (1+\sigma^2)P(d,k,\eps^{-1})\log(1/\delta),
\]
with the same prediction error, output width and coefficient
constraints. The real-operation count remains
$(1+\sigma^2)\poly(d,k,\eps^{-1},\log(1/\delta))$.
\end{corollary}

\begin{proof}
Run the learner independently $R=O(\log(1/\delta))$ times with target
squared error $\eps/9$ and fixed failure probability $1/20$.
Any run that terminates early returns the zero network, which satisfies
every output constraint. A Chernoff bound implies that, with probability at least
$1-\delta$, a strict majority of runs produce predictors within
$\sqrt\eps/3$ of $f$ in $L^2$. Those successful outputs have pairwise
distance at most $2\sqrt\eps/3$.

Using \eqref{cm:relu-kernel}, compute all pairwise squared $L^2$
distances between the outputs exactly. Select any output with a
strict majority of the $R$ outputs, counting itself, within distance
$2\sqrt\eps/3$. On the majority-success event such an output exists:
every successful output has this property. Any selected output has a
successful output in its neighborhood, since both the neighborhood
and the successful set contain a strict majority. Its distance from
$f$ is therefore at most $\sqrt\eps$ by the triangle inequality.
If no eligible output exists, return zero.

The selected predictor is one of the original outputs. It therefore
satisfies the same width, sign and mass constraints even if the selected
run was not itself successful.
Each run uses at most $(1+\sigma^2)P(d,k,\eps^{-1})$ samples for a
fixed polynomial $P$, and the comparisons require no further samples.
Their real-operation cost is $O(dk^2R^2)$.
\end{proof}

%% file: agnostic.tex
\section{Agnostic learning by positive ReLU networks}
\label{sec:agnostic}

The proof of Theorem~\ref{thm:agnostic} combines the robust direction
cover with penalized least squares over local tangent predictors. The
penalty controls the error when the fitted predictor is converted into
a positive ReLU network. The algorithm requires the supplied
second-moment bound $B$, but no knowledge of $g=\E[Y\mid X]$,
$\mathcal V$, $\mathcal A_k$, or the approximation error of a comparator.

Throughout this section, fix
\begin{equation}
 \theta=10^{-3},\qquad
 \eps_{\mathrm{int}}=\frac{\eps}{1000},\qquad
 \rho_*=\eps_{\mathrm{int}}^{1/3},\qquad
 \gamma_{\min}=\frac{\sqrt{\eps_{\mathrm{int}}}}{100k},\qquad
 b=\eps_{\mathrm{int}}.
 \label{ag:internal-parameters}
\end{equation}
Here $b$ is the tolerance for fitting and validation. Let $C_{\rm cov}$
be a universal constant satisfying the degree and perturbation
requirements of Proposition~\ref{prop:robust-cover}. Fix a universal
$C_{\rm noise}\geq4C_{\rm cov}$, and use moment degree
\begin{equation}
 t=\left\lceil C_{\rm deg}\log(2k/\eps_{\mathrm{int}})\right\rceil,
 \label{ag:degree}
\end{equation}
where $C_{\rm deg}$ is sufficiently large to meet that proposition's
degree requirement for every threshold at least $\gamma_{\min}$.
These constants are fixed independently of the problem parameters; the
constants in the complexity and approximation bounds may depend on them.

\subsection{A multiscale direction dictionary}
\label{ag:dictionary-section}

Fix a comparator $f=\sum_i a_i\relu(u_i^\top x)\in\mathcal F_k$, write
$g=f+e$, and put $\zeta=\norm{e}_{L^2}$. We first analyze the construction
for a parameter $\gamma\in[\gamma_{\min},1]$ satisfying
\begin{equation}
 C_{\rm noise}\sqrt{k}(t+1)^{9/4}\zeta\leq\gamma.
 \label{ag:valid-guess}
\end{equation}
This condition is needed only for the analysis. Set
\begin{equation}
 \eta_0=\frac{\gamma}{C_{\rm noise}\sqrt{k}(t+1)^{9/4}}.
 \label{ag:matrix-precision}
\end{equation}
Apply the coarse-cover procedure with threshold $\gamma$ and statistical
matrix error at most $\eta_0$. The perturbation bound in
Proposition~\ref{prop:robust-cover} is at most
$2C_{\rm cov}\gamma/C_{\rm noise}\leq\gamma/2$.
Prune the resulting list with parameter $\theta$ as in
Lemma~\ref{lem:pruning}. The list $W_0$ has at most $k$ centers and covers
every atom of weight at least $\gamma$ to radius
$\rho_0=3\theta$. Add the centers in $W_0$ to the dictionary, recording
radius $\rho_0$ for each.

Suppose that the current list has coverage radius $\rho\leq\rho_0$.
For each center $w$, apply the exact localization of
Lemma~\ref{lem:localization}. The transformed mean is
$g(TZ)=f(TZ)+e(TZ)$, and the change-of-measure identity gives
\begin{equation}
 \norm{e(TZ)}_{L^2}\leq\frac{\zeta}{\sqrt{\rho}},
 \qquad
 \E[Y^2\mid\mathrm{accepted}]\leq\frac{B}{\rho},
 \qquad
 a_i\|Tu_i\|\geq a_i\rho .
 \label{ag:localized-error}
\end{equation}
Indeed, the acceptance density is bounded by one and has integral
$\rho$, which proves the first two inequalities. Apply the
coarse-cover procedure to the localized distribution with threshold
and matrix precision
\begin{equation}
 \tau=\frac{\gamma}{\sqrt{\rho}},
 \qquad
 \eta_\rho\leq\frac{\eta_0}{\sqrt{\rho}}.
 \label{ag:local-threshold}
\end{equation}
If $\tau>1$, return an empty list for this call. Otherwise
\eqref{ag:valid-guess} and \eqref{ag:localized-error} imply the robust
cover condition. The threshold exceeds $\gamma_{\min}$, so
degree \eqref{ag:degree} suffices.

Retain the directions satisfying $|v^\top w|\geq2/3$ and map each back
by $T^{-1}v/\|T^{-1}v\|$. Lemma~\ref{lem:inverse} implies that every
mapped direction is within $3\theta\rho$ of an atom. An original atom
satisfying
\begin{equation}
 a_i\geq\gamma\rho^{-3/2}
 \label{ag:heavy-at-scale}
\end{equation}
has transformed weight at least $\tau$. If it is within $\rho$ of a
current center, the inverse lemma ensures that its coarse approximation
passes the filter and maps back within $3\theta\rho$.
Merge the mapped lists and prune at separation $6\theta\rho$. The next
list has at most $k$ centers and covers the atoms in
\eqref{ag:heavy-at-scale} to the new radius
\begin{equation}
 \rho_{\rm new}=9\theta\rho.
 \label{ag:radius-contraction}
\end{equation}
The set of atoms satisfying \eqref{ag:heavy-at-scale} shrinks as the
radius decreases, so each was covered by the preceding list. This
establishes the coverage invariant. Add the new centers to the
dictionary with radius $\rho_{\rm new}$.

Stop when the current radius satisfies $\rho^3\leq\eps_{\rm int}$
or the current list is empty. This is equivalent to $\rho\leq\rho_*$
and uses only arithmetic and comparisons; the algorithm need not
compute $\rho_*$. If $\rho_0^3\leq\eps_{\rm int}$, no localized call is
made. Impose the prescribed bounds on the coarse-cover computation and
rejection sampling, and a maximum list size of $k$. If any bound is exceeded,
return the zero predictor. Under \eqref{ag:valid-guess}, all bounds are
satisfied on the simultaneous success event. If no bound is exceeded,
the construction produces the multiscale dictionary
\begin{equation}
 \mathcal D_\gamma=\{(w_j,\rho_j):1\leq j\leq s\},
 \qquad
 s\leq Ck\log(2/\eps),\qquad 0<\rho_j\leq\rho_0.
 \label{ag:dictionary}
\end{equation}
Every localization uses fresh examples from the original distribution,
not previously accepted examples. Before a localized call,
$\rho>\rho_*$, so both the sampling overhead $\rho^{-1}$ and the label
moment bound $B/\rho$ remain polynomially bounded.

\subsection{Cubic tangent approximation and multiscale assignment}
\label{ag:tangent-section}

For unit vectors $u,w$, choose $v\in\{u,-u\}$ with
$\|v-w\|=D(u,w)$, and define
\begin{equation}
 T_{u,w}(x)=\frac12u^\top x+
             \frac12v^\top x\,\operatorname{sign}(w^\top x),
 \qquad \operatorname{sign}(0)=0.
 \label{ag:tangent}
\end{equation}
The residual is everywhere nonnegative:
\begin{equation}
 \relu(u^\top x)-T_{u,w}(x)
 =\frac12\bigl(|v^\top x|-v^\top x\,\operatorname{sign}(w^\top x)\bigr).
 \label{ag:tangent-residual}
\end{equation}
Off the Gaussian-null hyperplane $w^\top x=0$, it equals
$|v^\top x|\mathbf1_{\{(v^\top x)(w^\top x)<0\}}$.
If $\alpha\in[0,\pi/2]$ is the angle between $v$ and $w$, Gaussian polar
coordinates in their span give
\begin{equation}
 \|\relu(u^\top X)-T_{u,w}(X)\|_{L^2}^2
 =\frac{\alpha-\sin\alpha\cos\alpha}{\pi}
 \leq\frac{2\alpha^3}{3\pi}
 \leq D(u,w)^3.
 \label{ag:cubic-identity}
\end{equation}
There are two disagreement wedges of angle $\alpha$; integrating the
squared projection over them gives the equality. The first inequality
uses
$\alpha-\sin\alpha\cos\alpha=2\int_0^\alpha\sin^2 z\,dz\leq2\alpha^3/3$.
For the last inequality, use
$\alpha\leq(\pi/2)D(u,w)$ and $\pi^2/12<1$.
The collinear cases follow directly or by continuity.

To construct a feasible comparator, omit the atoms with $a_i<\gamma$
and assign each remaining atom to a nearest center at the last round
where the coverage invariant applies to it. Denote the retained index
set by $I$, the assigned center by $j(i)$, and the chosen orientation by
$\sigma_i\in\{+1,-1\}$. Thus
\begin{equation}
 \|\sigma_i u_i-w_{j(i)}\|\leq\rho_{j(i)}.
 \label{ag:assignment-angle}
\end{equation}
If an atom is no longer covered by the invariant after a round of radius
$\rho_{j(i)}$, the next threshold is
$\gamma\rho_{j(i)}^{-3/2}$, whence
$a_i\rho_{j(i)}^{3/2}<\gamma$. The same argument applies if the next list
is empty: on the success event, no atom then satisfies the next weight
threshold. Atoms covered by the invariant when the radius reaches
$\rho_*$ contribute at most $a_i\sqrt{\eps_{\rm int}}$.
Consequently,
\begin{equation}
 M_0:=\sum_{i\notin I}a_i\leq k\gamma,
 \qquad
 P_0:=\sum_{i\in I}a_i\rho_{j(i)}^{3/2}
       \leq k\gamma+\sqrt{\eps_{\rm int}}.
 \label{ag:assignment-budget}
\end{equation}
The bounds also cover an initially empty list. Minkowski's inequality,
\eqref{ag:cubic-identity}, and the unit-atom norm $1/\sqrt2$ imply
\begin{equation}
 \left\|f-\sum_{i\in I}a_iT_{u_i,w_{j(i)}}\right\|_{L^2}
 \leq M_0/\sqrt2+P_0
 \leq2k\gamma+\sqrt{\eps_{\rm int}}.
 \label{ag:tangent-comparator}
\end{equation}
The cubic exponent is essential: its $L^2$ exponent $3/2$ matches the
noise-dependent weight threshold in \eqref{ag:heavy-at-scale}.

\subsection{Cone constraints and conversion to positive networks}
\label{ag:cone-section}

For each center and radius in the dictionary, define the circular convex
cone
\begin{equation}
 \mathcal C_j=
 \{q=a w_j+z:\ z\perp w_j,\ a\geq0,\ \|z\|\leq2\rho_j a\}.
 \label{ag:cone}
\end{equation}
A unit vector $v$ with $\|v-w_j\|\leq\rho_j\leq1/10$ belongs to this cone:
$v^\top w_j\geq1-\rho_j^2/2$ and $\|v_\perp\|\leq\rho_j$.
Introduce variables $q_j^+,q_j^-\in\mathcal C_j$ and $r_j^+,r_j^-\geq0$
subject to
\begin{equation}
 \|q_j^\pm\|\leq r_j^\pm,\qquad
 \sum_{j,\sigma\in\{+,-\}}r_j^\sigma\leq1.
 \label{ag:feasible-set}
\end{equation}
These constraints define a compact convex set $K$. Associate with each
feasible point the tangent predictor and positive network
\begin{align}
 G_q(x)
 &=\frac12\sum_j
   \left[(q_j^+-q_j^-)^\top x
    +(q_j^++q_j^-)^\top x\,\operatorname{sign}(w_j^\top x)\right],
 \label{ag:gated-predictor}\\
 H_q(x)
 &=\sum_j\left[\relu((q_j^+)^\top x)+\relu(-(q_j^-)^\top x)\right].
 \label{ag:positive-predictor}
\end{align}
For every feasible point, $|G_q(x)|\leq\|x\|$ and
$0\leq H_q(x)\leq\|x\|$. Normalizing each nonzero vector in
\eqref{ag:positive-predictor} yields at most $2s$ unit directions with
nonnegative coefficients summing to at most one. Hence $H_q$ is
continuous, convex, nonnegative, and $1$-Lipschitz.

The tangent predictor is linear in its parameters. Define the
nonnegative linear penalty
\begin{equation}
 P(r)=\sum_{j,\sigma}\rho_j^{3/2}r_j^\sigma .
 \label{ag:radius-penalty}
\end{equation}

\begin{lemma}[Uniform conversion into a positive network]
\label{ag:conversion}
For every $(q,r)\in K$,
\begin{equation}
 H_q(x)-G_q(x)\geq0\quad\text{for every }x,\qquad
 \|H_q-G_q\|_{L^2}\leq A_0 P(r),\qquad A_0=3.
 \label{ag:conversion-bound}
\end{equation}
\end{lemma}

\begin{proof}
A nonzero $q\in\mathcal C_j$ has angle
$\alpha\leq\arctan(2\rho_j)\leq2\rho_j$ with $w_j$.
For either orientation, the corresponding summand of $H_q-G_q$ is
\[
 \tfrac12\left(|q^\top x|-q^\top x\,
                            \operatorname{sign}(w_j^\top x)\right)\geq0.
\]
With $\operatorname{sign}(0)=0$, this identity includes points on the
gate hyperplane; away from it the summand is
$|q^\top x|\mathbf1_{\{(q^\top x)(w_j^\top x)<0\}}$.
The zero-vector case is identically zero. The Gaussian calculation in
\eqref{ag:cubic-identity} gives squared norm at most
\[
 \|q\|^2\frac{2\alpha^3}{3\pi}
 \leq\frac{16}{3\pi}\|q\|^2\rho_j^3.
\]
Sum the $L^2$ bounds by Minkowski and use
$\|q_j^\sigma\|\leq r_j^\sigma$. Since
$\sqrt{16/(3\pi)}<3$, this proves the lemma.
\end{proof}

Cone-constrained convex formulations and conversion between gated and
ReLU models have precedents in
\citet{mishkin2022convex,dwaraknath2023ntk}.
The circular cones used here correspond to the refinement scales, and
\eqref{ag:conversion-bound} bounds the population error of the
conversion uniformly over the feasible set. The bound applies whether
or not the dictionary satisfies the coverage and soundness conditions.

\subsection{The penalized population objective}
\label{ag:objective-section}

Set $\lambda=A_0^2=9$ and consider
\begin{align}
 J(q,r)&=\E[(Y-G_q(X))^2]+\lambda P(r)^2,
 \label{ag:population-objective}\\
 F(q,r)&=J(q,r)-\mathcal V
        =\|g-G_q\|_{L^2}^2+\lambda P(r)^2.
 \label{ag:mean-objective}
\end{align}
Both objectives are convex in $(q,r)$ and have the same minimizers.
The conditional-mean decomposition and Lemma~\ref{ag:conversion} give
\begin{equation}
 \E[(Y-H_q(X))^2]
 =\mathcal V+\|g-H_q\|_{L^2}^2
 \leq\mathcal V+2F(q,r)
 =2J(q,r)-\mathcal V.
 \label{ag:objective-conversion}
\end{equation}
Thus the irreducible variance enters the risk bound with coefficient
one, without requiring an estimate of that variance.

For the assignment in \eqref{ag:assignment-angle}, define a feasible
point by
\begin{align}
 q_j^+&=\sum_{\substack{i\in I:j(i)=j\\\sigma_i=+1}}a_i u_i,
 &
 q_j^-&=\sum_{\substack{i\in I:j(i)=j\\\sigma_i=-1}}a_i(-u_i),
 \label{ag:comparison-vectors}\\
 r_j^+&=\sum_{\substack{i\in I:j(i)=j\\\sigma_i=+1}}a_i,
 &
 r_j^-&=\sum_{\substack{i\in I:j(i)=j\\\sigma_i=-1}}a_i .
 \label{ag:comparison-masses}
\end{align}
Each oriented unit vector lies in $\mathcal C_j$. Closure under
nonnegative linear combinations and the triangle inequality give
\eqref{ag:feasible-set}. Since $G_q$ is linear in the vectors $q$, it
equals the sum of the retained atoms' tangent predictors. Hence
\begin{equation}
 \|G_q-f\|_{L^2}\leq2k\gamma+\sqrt{\eps_{\rm int}},
 \qquad
 P(r)=P_0\leq k\gamma+\sqrt{\eps_{\rm int}}.
 \label{ag:feasible-comparator}
\end{equation}
Writing $A=\|g-f\|_{L^2}^2$ for this comparator and applying
$(a+b)^2\leq2a^2+2b^2$ gives
\begin{equation}
 F(q,r)\leq2A+34k^2\gamma^2+22\eps_{\rm int}.
 \label{ag:comparator-objective}
\end{equation}
Indeed, the two terms on the right of
$F\leq2A+2\|f-G_q\|_{L^2}^2+9P(r)^2$ contribute respectively
$16k^2\gamma^2+4\eps_{\rm int}$ and
$18k^2\gamma^2+18\eps_{\rm int}$.

\subsection{Finite-variance estimation of the convex objective}
\label{ag:estimation-section}

The parameter space has dimension $2s(d+1)$. The feasible set has
Euclidean diameter at most $2\sqrt2$, since
$\sum_{j,\sigma}(r_j^\sigma)^2\leq1$ and
$\sum_{j,\sigma}\|q_j^\sigma\|^2\leq1$.
The feature blocks for $G_q$ are
\begin{equation}
 \psi_j^+(X)=\tfrac12X(1+\operatorname{sign}(w_j^\top X)),
 \qquad
 \psi_j^-(X)=\tfrac12X(-1+\operatorname{sign}(w_j^\top X)).
 \label{ag:features}
\end{equation}
Their coordinates have magnitude at most $\|X\|$.
Population squared risk differs from
$\E[G_q^2]-2\E[YG_q]$ by the common constant $\E Y^2$.

Choose $R=\poly(d,\log(B/b))$ large enough that
\begin{equation}
 B\,\E[\|X\|^2\mathbf1_{\{\|X\|>R\}}]\leq c b^2,
 \label{ag:tail-cutoff}
\end{equation}
where $c>0$ is a sufficiently small universal constant. A Gaussian norm
tail bound supplies such an $R$. Put
$\psi_R=\mathbf1_{\{\|X\|\leq R\}}\psi$.
Estimate $\E[Y\psi_R]$ coordinatewise by median-of-means. Each coordinate
has variance at most $BR^2$, so a fourth label moment is unnecessary.
Using fresh unlabeled Gaussian samples, estimate the uncentered second
moment by the positive semidefinite matrix
\begin{equation}
 \widehat Q=\frac1N\sum_{n=1}^N
                \psi_R(X_n)\psi_R(X_n)^\top.
 \label{ag:empirical-gram}
\end{equation}
Its entries are bounded by $R^2$, so simultaneous entrywise estimation
has polynomial sample cost.

For every feasible point,
\[
 |\E[YG_q\mathbf1_{\{\|X\|>R\}}]|
 \leq\sqrt{B}\,
       \sqrt{\E[\|X\|^2\mathbf1_{\{\|X\|>R\}}]},
 \qquad
 \E[G_q^2\mathbf1_{\{\|X\|>R\}}]
 \leq\E[\|X\|^2\mathbf1_{\{\|X\|>R\}}].
\]
These bounds are uniform over $K$. Estimating the entries to accuracy
smaller than $b$ by a polynomial factor in the parameter dimension
therefore gives a convex quadratic estimate satisfying
\begin{equation}
 \sup_{(q,r)\in K}
 \left|\widehat J(q,r)-\bigl(J(q,r)-\E Y^2\bigr)\right|\leq b.
 \label{ag:uniform-estimate}
\end{equation}
The penalty is computed exactly. The estimation samples are independent
of those used to construct the dictionary. The Hessian of the estimated
quadratic part has operator norm at most a polynomial in $s,d,R$; the penalty Hessian
is $2\lambda aa^\top$ with entries
$a_{j,\sigma}=\rho_j^{3/2}$ and $\|a\|^2\leq2s$.
The gradient Lipschitz constant consequently has an explicit
polynomial bound.

\subsection{Convex optimization with an explicit linear oracle}
\label{ag:optimization-section}

For a linear objective $v^\top q+c r$ on one block, its minimum at
fixed $r\geq0$ is
\begin{equation}
 \min_{\substack{q\in\mathcal C_j\\\|q\|\leq r}}
       (v^\top q+c r)
 =r\bigl(c-\|\Pi_{\mathcal C_j}(-v)\|\bigr).
 \label{ag:block-oracle}
\end{equation}
For any closed convex cone $C$, Moreau's decomposition gives
$\max_{q\in C,\|q\|\leq1}z^\top q=\|\Pi_C z\|$:
the polar component has nonpositive inner product with $C$, and the
normalized cone projection attains this value. If the projection
vanishes, take $q=0$.

The projection in \eqref{ag:block-oracle} is explicit. Write a vector
to be projected as $z=a w_j+z_\perp$, put $u=\|z_\perp\|$ and
$\beta=2\rho_j$. Then
\begin{equation}
 \Pi_{\mathcal C_j}z=
 \begin{cases}
 z,&a\geq0,\ u\leq\beta a,\\
 0,&a+\beta u\leq0,\\
 h w_j+\beta h\,z_\perp/u,
    &\text{otherwise},\quad h=(a+\beta u)/(1+\beta^2).
 \end{cases}
 \label{ag:cone-projection}
\end{equation}
In the last case $u>0$. The cases with $u=0$ belong to the first two
alternatives. To verify the formula, rotational symmetry reduces the
projection to minimizing $(h-a)^2+(z'-u)^2$ over
$h\geq0$, $0\leq z'\leq\beta h$; in the third region the minimizer lies
on $z'=\beta h$.

To minimize a linear function over $K$, compute the coefficient in
\eqref{ag:block-oracle} for every block. If their minimum is negative,
put $r=1$ in a minimizing block, choose its minimizing $q$, and set all
other blocks to zero. If all coefficients are nonnegative, return the
origin. The resulting exact oracle uses polynomially many arithmetic,
comparison, and square-root operations.

Apply conditional gradient to $\widehat J$, starting at the origin, with
step sizes $2/(n+2)$ for $n=0,1,\ldots$. In particular the first step
has size one. The smoothness inequality and the diameter bound imply
objective gap $O(L/(n+1))$ after $n\geq1$ steps, where $L$ is the
polynomial bound on the gradient Lipschitz constant. The first full step
removes any dependence of this bound on the initial objective gap. Polynomially many
steps therefore yield optimization error at most $b$.
When \eqref{ag:uniform-estimate} holds, the resulting feasible point
satisfies
\begin{equation}
 F(\widehat q,\widehat r)
 \leq\inf_{(q,r)\in K}F(q,r)+3b.
 \label{ag:fitting-error}
\end{equation}
Subtracting the unknown constant $\mathcal V$ changes neither this
inequality nor the optimization procedure.
Define the output network $H_{\widehat q}$ by
\eqref{ag:positive-predictor}. Under \eqref{ag:valid-guess} and the
simultaneous success event,
\eqref{ag:objective-conversion}, \eqref{ag:comparator-objective}, and
\eqref{ag:fitting-error} give
\begin{equation}
 \E[(Y-H_{\widehat q}(X))^2]
 \leq\mathcal V+4A+68k^2\gamma^2+50\eps_{\rm int}.
 \label{ag:single-guess-risk}
\end{equation}
Every feasible fit satisfies the stated sign, mass, width, continuity,
convexity, and Lipschitz constraints, even if
\eqref{ag:valid-guess} or a statistical guarantee fails.

\subsection{Scale selection and the final risk bound}
\label{ag:selection-section}

Construct a dictionary and fit a predictor, subject to the prescribed
computational bounds, for each value in the dyadic grid
\begin{equation}
 \{\gamma_{\min}2^j:j\geq0,\ \gamma_{\min}2^j\leq1\}\cup\{1\},
 \label{ag:gamma-grid}
\end{equation}
and include the zero network among the candidates. The grid has
$O(\log(2k/\eps))$ values. Model selection is based on the risk of the
converted networks $H$. For each candidate, estimate
\begin{equation}
 \E[H(X)^2]-2\E[YH(X)]
 \label{ag:validation-score}
\end{equation}
to additive accuracy $b$ using fresh samples. The same cutoff
\eqref{ag:tail-cutoff}, the bound $0\leq H(X)\leq\|X\|$, a
finite-variance estimate for the truncated cross term, and a bounded
estimate for the truncated quadratic term suffice. A union bound gives
simultaneous estimates for all candidates. Minimizing the estimated
score selects a network whose risk is at most $2b$ above the smallest
candidate risk. Neither $\E Y^2$ nor $\mathcal V$ needs to be estimated.

\begin{proof}[Proof of Theorem~\ref{thm:agnostic}]
The parameterization of $\mathcal F_k$ by a compact weight simplex and
$k$ unit spheres is compact, and its map into Gaussian $L^2$ is
continuous. Hence a minimizer $f$ of $\|g-f\|_{L^2}^2$ exists.
Fix such a minimizer, so $A=\mathcal A_k$ and
$\zeta=\sqrt{\mathcal A_k}$. Set
\[
 S=C_{\rm noise}\sqrt{k}(t+1)^{9/4}\sqrt{\mathcal A_k}.
\]
If $S\leq1$, the grid contains a parameter satisfying
\eqref{ag:valid-guess} and
$\gamma\leq2\max\{S,\gamma_{\min}\}$. After validation,
\eqref{ag:single-guess-risk} implies
\begin{align}
 \E[(Y-H(X))^2]
 &\leq\mathcal V+4\mathcal A_k+
       272 C_{\rm noise}^2 k^3(t+1)^{9/2}\mathcal A_k
       +53\eps_{\rm int}\notag\\
 &\leq\mathcal V+Ck^3(t+1)^{9/2}\mathcal A_k+\eps.
 \label{ag:small-approximation}
\end{align}
Here $272k^2\gamma_{\min}^2=272\eps_{\rm int}/10^4$, the
validation error is $2b=2\eps_{\rm int}$, and
\eqref{ag:internal-parameters} absorbs their sum with the
$50\eps_{\rm int}$ term.

If $S>1$, the zero network satisfies
\begin{equation}
 \E Y^2=\mathcal V+\|g\|_{L^2}^2
 \leq\mathcal V+2\mathcal A_k+2\|f\|_{L^2}^2
 \leq\mathcal V+2\mathcal A_k+1.
 \label{ag:large-approximation}
\end{equation}
The condition $S>1$ gives
$k(t+1)^{9/2}\mathcal A_k>C_{\rm noise}^{-2}$. Increasing the
universal constant absorbs the additive one into
$Ck^3(t+1)^{9/2}\mathcal A_k$; validation adds only $2b<\eps$.
Since $t+1=O(\log(2k/\eps))$, both cases prove
\begin{equation}
 \E[(Y-H(X))^2]
 \leq\mathcal V+
       Ck^3\log^{9/2}(2k/\eps)\mathcal A_k+\eps.
 \label{ag:oracle-risk}
\end{equation}

To bound the complexity, note that every localized call has
$\rho^{-1}<\eps_{\rm int}^{-1/3}$, every nontrivial coarse threshold
is at least $\gamma_{\min}$, and the inverse matrix precisions can be
bounded by a fixed polynomial in $k,\eps^{-1}$. Equation
\eqref{ag:degree} makes $\exp(O(t))$ a fixed polynomial. There are
polynomially many grid values, at most $k$ localized calls per round, and
$O(\log(2/\eps))$ rounds. The per-call bound $B/\rho$, rejection
overhead, tensor calculations, objective estimation, convex fitting,
and validation therefore all have total sample and real-operation
cost polynomial in $d,k,B,\eps^{-1},\log(1/\delta)$.

Use a deterministic upper bound on the number of calls and estimates
to allocate their failure probabilities with total at most $\delta$.
Each adaptive call uses fresh samples after its parameters are fixed;
conditioning on the past and summing these conditional failure
probabilities justifies the union bound. Rejection sampling with the
prescribed sample limit is coupled to the infinite accepted sequence as in
Lemma~\ref{lem:localization}, counting an exceeded sampling bound as
failure. On the simultaneous success event, all estimates required for
the parameter used above are accurate. For every grid value, the
computation obeys the prescribed bounds and returns either a feasible
network or zero. Finally \eqref{ag:dictionary} and
\eqref{ag:positive-predictor} give
$m\leq2s\leq Ck\log(2/\eps)$ and total output weight at most one,
with the stated continuity, convexity, nonnegativity, and Lipschitz
properties.
\end{proof}

\begin{corollary}[Risk relative to the best comparator network]
\label{ag:opt-corollary}
Under the assumptions of Theorem~\ref{thm:agnostic}, put
$\mathrm{OPT}=\inf_{f\in\mathcal F_k}\E[(Y-f(X))^2]$.
The same positive-network output and complexity guarantees hold with
\begin{equation}
 \E[(Y-H(X))^2]\leq
 Ck^3\log^5(2k/\eps)\,\mathrm{OPT}+\eps .
 \label{ag:opt-bound}
\end{equation}
If $g\in\mathcal F_k$, then
$\|H-g\|_{L^2}^2\leq\eps$, irrespective of the size of the
irreducible variance, subject to the supplied second-moment bound $B$.
\end{corollary}

\begin{proof}
The conditional-mean decomposition gives
$\mathrm{OPT}=\mathcal V+\mathcal A_k$. Apply
\eqref{ag:oracle-risk} and absorb universal constants, including the
change from logarithmic exponent $9/2$ to $5$. If $g\in\mathcal F_k$,
then $\mathcal A_k=0$, and subtracting $\mathcal V$ proves the last
statement.
\end{proof}

\begin{corollary}[Gated-linear output]
\label{ag:gated-corollary}
With the same sample and operation bounds, the algorithm can also return
\begin{equation}
 G(x)=\ell^\top x+\sum_{j=1}^{s}b_j^\top x\,
                     \operatorname{sign}(w_j^\top x),
 \qquad
 s\leq Ck\log(2/\eps),\qquad
 \|\ell\|+\sum_j\|b_j\|\leq1,
 \label{ag:gated-output}
\end{equation}
with risk bounded by the right side of \eqref{ag:oracle-risk}, and hence
also by \eqref{ag:opt-bound}. This output can be discontinuous at its
gate hyperplanes.
\end{corollary}

\begin{proof}
Retain the tangent function of every fitted grid point and use
an independent validation sample to select among these functions and
zero. For each fitted point, take
$\ell=\tfrac12\sum_j(q_j^+-q_j^-)$ and
$b_j=\tfrac12(q_j^++q_j^-)$; the triangle inequality gives the norm
constraint in \eqref{ag:gated-output}. Since
$|G_q(X)|\leq\|X\|$, the same validation method applies.
For the parameter satisfying \eqref{ag:valid-guess},
$\E[(Y-G_{\widehat q}(X))^2]
\leq\mathcal V+F(\widehat q,\widehat r)$.
Use \eqref{ag:comparator-objective} and \eqref{ag:fitting-error}, then
the same two cases for $\mathcal A_k$. The resulting error bounds are
no larger than those established for the positive network.
\end{proof}

%% file: alternative_cover.tex
\section{Direction recovery by symmetric tensor compression and reweighting}
\label{sec:alternative}

We give a direction-cover construction for the conditional-mean model
\eqref{eq:conditional-model} using symmetric tensor compression and
complex Gaussian reweighting. The construction estimates Gram matrices
to obtain approximately orthonormal tensor bases, then applies random
reweightings to isolate neighborhoods of the hidden directions. Its
sample and computational complexity are polynomial in the model
parameters and linear in one plus the noise second moment.

Set
\begin{equation}
 \theta=10^{-3},\qquad \eta_*=\theta^2/100.
 \label{alt:parameters}
\end{equation}
We write $\zeta_G$ for the Gram-matrix error. Tensor polynomials,
defined in Section~\ref{sec:moments}, are represented by arithmetic
circuits that share subexpressions across linear combinations and
products. The scalar estimation argument in
Lemma~\ref{lem:query-access} applies to any such bounded-norm tensor
query after rescaling; orthonormality is needed there only for the
Frobenius bound on deterministic mean perturbations. Complex queries
are estimated through their real and imaginary parts. Unlabeled Gram
queries use the product identities following
Lemma~\ref{lem:moment-signal}.

\subsection{Stable symmetric tensor compression}

\begin{lemma}[Recursive tensor compression]
\label{alt:compression}
Given an integer $t\geq1$, $\kappa\in(0,1)$, and a failure probability
$\omega\in(0,1)$, one can construct a map $B$ with at most $k$ columns,
each a real symmetric order-$t$ tensor, such that, with probability at least
$1-\omega$,
\begin{equation}
 \norm{B^*B-I}_{\op}\leq\zeta_G\leq\kappa,
 \qquad
 \sum_i a_i\dist(u_i^{\otimes t},\range B)^2\leq\kappa^2.
 \label{alt:compression-guarantee}
\end{equation}
The column polynomials can be evaluated by arithmetic circuits with
shared subexpressions. The sample complexity and number of
real-arithmetic operations are at most
\begin{equation}
 (1+\sigma^2)\exp(O(t))\poly\!\left(d,k,\kappa^{-1},
                                      \log(1/\omega)\right).
 \label{alt:compression-cost}
\end{equation}
The empty map is permitted, with the Gram condition interpreted vacuously.
\end{lemma}

\begin{proof}
Begin with the order-zero basis $\{1\}$. Suppose a map $V$ with at most $k$
nearly orthonormal columns has been constructed at order $j-1$. Form the
$n\leq dk$ candidate columns
\begin{equation}
 A_{s\ell}=\Sym(e_s\otimes v_\ell),\qquad
 p_{A_{s\ell}}(z)=z_s p_{v_\ell}(z).
 \label{alt:candidates}
\end{equation}
Estimate $A^*A$ to operator error $\eta_G$. Discard estimated eigenvalues
smaller than a threshold $\alpha>0$ and whiten using the retained estimated
eigenvalues. If $U\Lambda U^*$ is the retained eigendecomposition, put
$B=AU\Lambda^{-1/2}$. Then
\begin{equation}
 \norm{B^*B-I}_{\op}\leq\eta_G/\alpha\leq\zeta_G
 \label{alt:whitening}
\end{equation}
provided $\eta_G\leq\zeta_G\alpha$.

Let $e_{j-1}$ be the square root of the weighted mean-square projection error
at order $j-1$. For each unit $u$, write the projection of
$u^{\otimes(j-1)}$ onto $\range V$ as $Vc(u)$. Near orthonormality implies
$\norm{c(u)}_2\leq(1-\zeta_G)^{-1/2}$. Tensoring with $u$ and symmetrizing
does not increase the approximation error. The candidate coefficient vector
$u\otimes c(u)$ has the same norm bound. In the discarded estimated Gram
subspace, the true quadratic form of $A^*A$ is at most
$\alpha+\eta_G$. Hence the additional approximation error from discarding
that subspace is at most
\begin{equation}
 \sqrt{\frac{\alpha+\eta_G}{1-\zeta_G}}.
 \label{alt:gram-discard}
\end{equation}
This argument requires no lower singular-value bound for the unknown atom
dictionary.

Set $z(u)=B^*u^{\otimes j}$ and
$C=\sum_i a_i z(u_i)z(u_i)^*$. This positive semidefinite matrix has rank at
most $k$. Its entries are order-$2j$ atomic queries and can be estimated to
operator error $\xi$ using \eqref{eq:merged-atomic-moment}. Retain the eigenvectors
of the estimated matrix with eigenvalues at least $\beta>\xi$, and let $R$
have these orthonormal columns. The min--max principle implies that $R$ has
at most $k$ columns on the estimation success event. If more than $k$
columns are retained, the algorithm terminates and reports failure.
Thus the dimensions of all recursively constructed matrices remain
bounded even when estimation fails.
Let $Q=I-RR^*$. The matrix $QCQ$ is positive semidefinite, has rank at
most $k$, and satisfies $\norm{QCQ}_{\op}\leq\beta+\xi$. Consequently,
\begin{equation}
 \tr[(I-RR^*)C]\leq k(\beta+\xi).
 \label{alt:cov-discard}
\end{equation}
The updated tensor basis is $BR$.

For completeness, put $H=B^*B$ and let
$P=B H^{-1}B^*$ be the orthogonal projector onto $\range B$. Since the
nonzero eigenvalues of $BB^*$ equal those of $H$,
$\norm{P-BB^*}_{\op}\leq\zeta_G$.
Use $BRR^*B^*u^{\otimes j}$ as a comparison vector in $\range(BR)$.
The triangle inequality in weighted $L^2$, \eqref{alt:gram-discard}, and
\eqref{alt:cov-discard} give
\begin{equation}
 e_j\leq e_{j-1}
       +\sqrt{\frac{\alpha+\eta_G}{1-\zeta_G}}
       +\zeta_G
       +\sqrt{(1+\zeta_G)k(\beta+\xi)}.
 \label{alt:compression-recursion}
\end{equation}
We used $\sum_i a_i\leq1$ for the uniform error terms.
The new Gram matrix also has spectrum in $[1-\zeta_G,1+\zeta_G]$.

Fix $C_0=100$ and choose, at all levels,
\begin{equation}
 \alpha=\frac{\kappa^2}{C_0t^2},\qquad
 \zeta_G=\frac{\kappa}{C_0t},\qquad
 \beta=\frac{\kappa^2}{C_0t^2k},\qquad
 \eta_G\leq\frac{\zeta_G\alpha}{4},\qquad \xi\leq\frac{\beta}{4}.
 \label{alt:compression-parameters}
\end{equation}
Each increment in \eqref{alt:compression-recursion} is then at most
$\kappa/t$, so $e_t\leq\kappa$. If a basis becomes empty, it remains empty,
and the same error bound applies.

Every required inverse precision is polynomial in $d,k,t,\kappa^{-1}$.
The largest labeled query order is $2t$. Candidate tensors and whitened
columns have bounded Frobenius norms on the preceding success events, so
the moment bounds give \eqref{alt:compression-cost} after allocating the
failure probability among all entries and levels. At each level, linear
combinations and the multiplications in \eqref{alt:candidates} add only a
polynomial number of circuit gates. Sharing the gates from earlier
levels keeps the total circuit size polynomial. Each stage uses an
independent estimation batch, so the bounds hold conditionally on
successful estimates at the preceding stages.
\end{proof}

\subsection{A direction cover from complex Gaussian reweighting}
\label{alt:coarse-section}

\begin{proposition}[Complete and sound coarse cover]
\label{alt:coarse}
Fix $\theta$ and $\eta_*$ as in \eqref{alt:parameters}. For any
$\tau\in(0,1)$ and failure probability $\omega\in(0,1)$, a randomized
procedure returns a finite list $L\subseteq\sphere$ such that, with
probability at least $1-\omega$:
\begin{enumerate}
\item for every atom with $a_i\geq\tau$, some $v\in L$ satisfies
      $D(v,u_i)\leq\theta$ \emph{(completeness)};
\item every $v\in L$ satisfies $D(v,u_i)\leq\theta$ for at least one
      positive-weight atom $u_i$ \emph{(soundness)}.
\end{enumerate}
The list length is bounded by a fixed polynomial in
$k,\tau^{-1},\log(2/\omega)$. The sample complexity and number of real-arithmetic operations are at most
$(1+\sigma^2)\poly(d,k,\tau^{-1},\log(2/\omega))$.
The witness atom in the soundness statement need not be heavy.
\end{proposition}

\subsubsection{Concentration under random reweighting}

Let $g=X_0+\ii Y_0$, where $X_0,Y_0$ are independent standard Gaussian
vectors. Fix $C_1=32$ and $c_1=1/100$, and set
\begin{equation}
 t=\left\lceil C_1\eta_*^{-1}\log\frac{100k}{\tau\eta_*}\right\rceil,
 \qquad R^2=16t,\qquad p=c_1e^{-8t},\qquad m_0=\tau8^t.
 \label{alt:reweight-parameters}
\end{equation}
Because $\eta_*$ is fixed, $\exp(O(t))$, $p^{-1}$, $m_0$, and $m_0^{-1}$
are bounded by fixed polynomials in $k$ and $\tau^{-1}$.

Fix a target atom $u$ of weight at least $\tau$, and put $z=g^\top u$.
For any other atom, write $r_i=u^\top u_i$. Conditional on $z$,
\begin{equation}
 g^\top u_i=r_i z+\sqrt{1-r_i^2}\,\xi_i,
 \label{alt:conditional-complex}
\end{equation}
where each nondegenerate $\xi_i$ has independent standard Gaussian real and
imaginary parts. The $\xi_i$ need not be independent across atoms.
Call an atom far if $1-r_i^2\geq\eta_*$. For such an atom,
\begin{equation}
 \Prob\!\left(\abs{g^\top u_i}>(1-\eta_*/4)\abs{z}\,\middle|\,z\right)
 \leq \exp\!\left(-\frac{\eta_*\abs{z}^2}{32}\right).
 \label{alt:far-tail}
\end{equation}
Indeed, it suffices that
$\abs{\xi_i}\leq(1-\abs{r_i})\abs{z}/(2\sqrt{1-r_i^2})$.
This threshold is at least $\sqrt{\eta_*}\abs{z}/4$, and the radial tail of
$\xi_i$ is $\Prob(|\xi_i|>s)=e^{-s^2/2}$.

The event $R^2\leq|z|^2\leq R^2+2$ has probability
$(1-e^{-1})e^{-8t}$. Conditional on it, a union bound in
\eqref{alt:far-tail} ensures, with probability at least $0.99$, that all far
atoms simultaneously satisfy the amplitude bound. By the choice of $c_1$,
this joint target event has probability at least $4p$.

Define
\begin{equation}
 q_g(v)=\frac{(g^\top v)^t}{\sqrt{2^t t!}},\qquad
 Z_g=\sum_i a_i|q_g(u_i)|^2,\qquad
 C_g=\frac{\sum_i a_i|q_g(u_i)|^2u_iu_i^\top}{Z_g}.
 \label{alt:ideal-weight}
\end{equation}
The normalized matrix $C_g$ is defined whenever $Z_g>0$.
On the target event,
$Z_g\geq\tau R^{2t}/(2^t t!)\geq m_0$.
The total normalized weight on far atoms is at most
\begin{equation}
 \tau^{-1}(1-\eta_*/4)^{2t}\leq\eta_*/100.
 \label{alt:far-mass}
\end{equation}
It follows that $C_g\succeq0$, $\tr C_g=1$, and
$u^\top C_g u\geq1-2\eta_*$. A leading eigenvector is projectively close to $u$;
below we establish this conclusion directly for the estimated matrix used
by the algorithm.

\subsubsection{Implicit reweighting with a deterministic norm bound}

The norm of $g^{\otimes t}$ is too large for the required moment estimates.
We therefore apply Lemma~\ref{alt:compression} with a precision
\begin{equation}
 \kappa=c_\kappa
 \min\!\left\{\frac14,\frac{\eta_* m_0p}{10^8}\right\},
 \label{alt:kappa-reweight}
\end{equation}
where $c_\kappa=1/100$.
Condition on successful compression, put $F=BB^*$, and define
\begin{equation}
 T_g=\frac{Fg^{\otimes t}}{\sqrt{2^t t!}}.
 \label{alt:compressed-reweight}
\end{equation}
This is a linear combination of the real basis columns $b_\ell$ with
coefficients $p_{b_\ell}(g)/\sqrt{2^t t!}$, so its polynomial is directly
evaluated by their shared circuits.

Complex Gaussian isotropy on symmetric tensors gives
\begin{align}
 \E\norm{T_g}_{\Frob}^2&=\tr(F^2)\leq2k,
 & \E Z_g&\leq1, \label{alt:compressed-moments}\\
 \E\sum_i a_i|q_g(u_i)-p_{T_g}(u_i)|^2
 &=\sum_i a_i\norm{(I-F)u_i^{\otimes t}}_{\Frob}^2
 \leq\kappa^2+\zeta_G^2\leq2\kappa^2. \label{alt:compressed-error}
\end{align}
For the last bound, decompose $I-F$ as $(I-P)+(P-F)$, where $P$ projects
orthogonally onto $\range B$. The two resulting vectors are orthogonal,
and $\norm{P-F}_{\op}\leq\zeta_G$.

Set $C_2=1000$. Markov's inequality shows that
\begin{equation}
 \norm{T_g}_{\Frob}^2\leq C_2k/p,\quad
 Z_g\leq C_2/p,\quad
 E_g:=\sum_i a_i|q_g(u_i)-p_{T_g}(u_i)|^2\leq C_2\kappa^2/p
 \label{alt:markov-event}
\end{equation}
fail with total probability at most $5p/C_2=p/200$.
Moreover, Cauchy--Schwarz gives
\begin{align}
 \sum_i a_i\left||q_g(u_i)|^2-|p_{T_g}(u_i)|^2\right|
 &\leq \sqrt{E_g}\left(2\sqrt{Z_g}+\sqrt{E_g}\right) \nonumber\\
 &\leq 3000\kappa/p\leq\eta_* m_0/100,
 \label{alt:weight-stability}
\end{align}
by \eqref{alt:kappa-reweight}.
Thus both the mass and the unnormalized covariance change by at most the
last quantity. On the intersection with the target event, the compressed
mass is at least $0.99m_0$, and its normalized covariance $C$ satisfies
$u^\top C u\geq1-3\eta_*$.

Conditional on successful compression, the following rejection rule gives a
deterministic norm bound for every retained query. The algorithm computes
\begin{equation}
 S_g=(1+\zeta_G)\sum_\ell\frac{|p_{b_\ell}(g)|^2}{2^t t!}
 \label{alt:norm-screen}
\end{equation}
and rejects any proposal with $S_g>C_2k/p$. Since
$\norm{T_g}_{\Frob}^2\leq S_g$ and
$\E S_g=(1+\zeta_G)\tr(B^*B)\leq2k$, this rejection costs at most
$2p/C_2$ in probability. One may also reject proposals with
$\norm{g}_2^2>100(d+t+1)$. For $\norm{g}_2^2\sim\chi^2_{2d}$,
Markov's inequality applied to $\exp(\norm{g}_2^2/4)$ bounds this
probability by $2^d e^{-25(d+t+1)}<p/100$.
The target event has probability at least $4p$, while the three bounds in
\eqref{alt:markov-event} and the two screens lose less than $p$ in total.
Their intersection therefore has probability at least $p$.

\subsubsection{Acceptance tests and soundness of every output}

For each retained tensor $T=T_g$, estimate
\begin{equation}
 z_T=\sum_i a_i|p_T(u_i)|^2,\qquad
 A_T=\sum_i a_i|p_T(u_i)|^2u_iu_i^\top.
 \label{alt:query-weighted}
\end{equation}
These are atomic queries of orders $2t$ and $2t+2$. The appropriate
holomorphic polynomial for the squared weight is
\begin{equation}
 p_T(Z)p_{\overline T}(Z),
 \label{alt:holomorphic-square}
\end{equation}
which conjugates the coefficients only. Evaluating $|p_T(Z)|^2$ at the
complex query argument would not give the required Hermite query.
The tensor in \eqref{alt:holomorphic-square} has Frobenius norm at most
$\norm{T}_{\Frob}^2$, as do the coordinate-augmented tensors for entries of
$A_T$. On successful compression, the screen bounds these tensor norms,
and Lemma~\ref{lem:query-access} supplies finite-variance estimates under
the global noise assumption. The same sample sizes are used when compression fails, which keeps
the running time bounded on every outcome.

Using fresh labeled samples, construct a scalar $\widehat z_T$ and a
symmetric matrix $\widehat A_T$ such that
\begin{equation}
 |\widehat z_T-z_T|\leq e_0:=\frac{\eta_* m_0}{10^4},\qquad
 |(\widehat A_T-A_T)_{ab}|\leq\frac{e_0}{d}.
 \label{alt:weighted-estimation}
\end{equation}
The latter controls both the operator error and the trace error by $e_0$.
Accept precisely when
\begin{equation}
 \widehat z_T\geq m_0/2,\qquad
 \lambda_{\max}(\widehat A_T/\widehat z_T)\geq1-5\eta_*,
 \label{alt:accept}
\end{equation}
and output a unit leading eigenvector $v$ of $\widehat A_T$.

We first prove soundness. For any accepted proposal with accurate estimates,
$z_T\geq0.49m_0$. Since $A_T\succeq0$ and $\tr A_T=z_T$,
\begin{equation}
 \left\lVert\frac{\widehat A_T}{\widehat z_T}
                   -\frac{A_T}{z_T}\right\rVert_{\op}
 \leq\frac{4\eta_*}{10^4}.
 \label{alt:normalized-perturbation}
\end{equation}
Consequently $v^\top(A_T/z_T)v\geq1-6\eta_*$. This Rayleigh quotient is a
positive weighted average of $|v^\top u_i|^2$. At least one positive-weight
atom therefore satisfies $|v^\top u_i|^2\geq1-6\eta_*$, and
\begin{equation}
 D(v,u_i)^2=2(1-|v^\top u_i|)\leq12\eta_*<\theta^2.
 \label{alt:soundness}
\end{equation}

For completeness, on a target event surviving the screens,
$z_T\geq0.99m_0$ and $u^\top C u\geq1-3\eta_*$, where $C=A_T/z_T$.
The estimation bounds imply both tests in \eqref{alt:accept}. The returned
vector satisfies $v^\top C v\geq1-6\eta_*$. As $C\succeq0$ and
$\tr C=1$,
\begin{equation}
 2-9\eta_*\leq u^\top C u+v^\top C v
 =\tr[C(uu^\top+vv^\top)]
 \leq1+|u^\top v|.
 \label{alt:completeness-rayleigh}
\end{equation}
Thus $D(u,v)^2\leq18\eta_*<\theta^2$.

Draw
\begin{equation}
 N=\left\lceil p^{-1}\log\frac{3k}{\omega}\right\rceil
 \label{alt:proposal-count}
\end{equation}
independent proposals. Conditional on successful compression, each heavy
target has a probability at least $p$ of obtaining
a successful proposal. A union bound covers all heavy targets. Distribute
the remaining failure probability among compression and the estimation of all
retained proposals. The bounds on $t$, $p$, $m_0$, $\kappa$, and the query
norms show that every inverse precision and sample count is polynomial in
the parameters stated in Proposition~\ref{alt:coarse}. This proves that
proposition. The argument establishes soundness for all outputs, including
those whose witnesses have arbitrarily small positive weights.